\documentclass{article} 
\usepackage{iclr2027_conference,times}

\usepackage{amsmath,amsfonts,bm}

\def\eqref#1{equation~\ref{#1}}

\def\1{\bm{1}}

\DeclareMathAlphabet{\mathsfit}{\encodingdefault}{\sfdefault}{m}{sl}
\SetMathAlphabet{\mathsfit}{bold}{\encodingdefault}{\sfdefault}{bx}{n}

\usepackage{svg} 
\usepackage{hyperref}
\usepackage{url}
\usepackage{graphicx}
\usepackage[table]{xcolor}
\usepackage{booktabs}
\usepackage{graphicx}
\usepackage{wrapfig}
\usepackage{subcaption}
\usepackage{capt-of} 
\usepackage{algorithm}
\usepackage{algorithmic}
\usepackage{amsthm}
\newtheorem{proposition}{Proposition}

\usepackage{multirow}
\usepackage{array}
\usepackage{caption}
\usepackage{enumitem}

\usepackage{cleveref}

\usepackage{tikz}

\hypersetup{
	colorlinks=true,
	urlcolor=purple,
	citecolor=purple,
}

\newtheorem{lemma}{Lemma}[section]

\newcommand{\meanstd}[2]{%
  \makebox[4.5em][c]{%
    #1%
    \makebox[0pt][l]{%
      \hspace{0.15em}%
      \raisebox{-0.42ex}[0pt][0pt]{%
        \scalebox{0.60}{%
          {\bfseries\boldmath
          \textcolor{gray!70!black}{\(\pm\,#2\)}}%
        }%
      }%
    }%
  }%
}

\title{\textbf{GraphOPD}:
       Graph-Augmented On-Policy Distillation for LLM Agents}

\author{
Bohan~Lin$^{1,2}$\thanks{Work done during an internship at Xiaohongshu Inc.}, \hspace{0.4em}Liyi~Chen$^2$\textsuperscript{$\dagger$}, \hspace{0.4em}Zhuoning~Guo$^3$,\hspace{0.4em} Muyang~Li$^1$, \hspace{0.4em}Qimeng~Wang$^2$, \\[4pt]
\textbf{Yan~Gao$^2$,\hspace{0.4em} Yao~Hu$^2$, \hspace{0.4em}Yudong~Zhang$^1$}
\thanks{Corresponding authors.}
\\[6pt]
$^1$ University of Science and Technology of China,\quad $^2$ Xiaohongshu Inc.\hspace{0.4em},\\
$^3$ The Hong Kong University of Science and Technology (HKUST-GZ)
\\[4pt]
\texttt{linbohan@ustc.edu; liyichencly@gmail.com; yudong.zhang@ustc.edu.cn}
}

\iclrfinalcopy 
\begin{document}

\maketitle

\begin{abstract}
On-policy distillation post-trains large language model agents by
supplying dense, step-level guidance from a teacher policy when the
reinforcement-learning reward is sparse and arrives only once per
trajectory.
Existing instantiations allocate this guidance by the size of the
teacher-student divergence at each step, on the single-turn intuition
that a large disagreement marks a mistake worth correcting.
Once decisions chain over many turns, that rule misfires, since an early
drift enters every later context both policies condition on, leaving the
teacher consistent with the drifted trajectory instead of flagging its
cause, while interchangeable steps register large but outcome-irrelevant
divergences.
We demonstrate this on an agentic benchmark, where distilling the
highest-divergence steps brings no consistent benefit over random
selection.
To this end, we introduce \textbf{GraphOPD}, the first method to bring graph-based
structural augmentation into on-policy distillation for agent
capabilities.
It reads which steps enabled which later ones from the environment's own record of state changes, immune to the drift that corrupts the teacher-student gap, organizes them into a dependency graph, scores each step by a random-walk stationary distribution over it, and fuses that structural credit with the divergence signal into a trajectory-relative mask concentrating supervision on each rollout's highest-aptitude steps.
Across three model scales and eleven baselines on ALFWorld, WebShop, and
SearchQA, GraphOPD shows competitive performance throughout, improving
over the strongest baseline by up to \textbf{+5.8 pp}.
An executed-replay audit further shows that this structural credit score
tracks true causal impact far above chance, that both fused signals are
independently necessary, and that the same signal transfers to
out-of-domain tool-integrated reasoning.The code is at \url{https://github.com/BoHan-LIN04/GraphOPD}.
\end{abstract}

\section{Introduction}
\label{sec:intro}

On-policy distillation has become a standard way to post-train large
language model agents that pursue goals through extended interaction with
an environment~\citep{yao2023react}, spanning household task execution,
multi-step retrieval, and web
purchasing~\citep{shridhar2021alfworld,yao2022webshop,yang2018hotpotqa}.
Reinforcement learning provides the outcome-level training signal for
such agents~\citep{grpo2024}, and because that signal is sparse and
arrives only once per trajectory, on-policy distillation complements it
with dense, step-level supervision from a teacher policy, typically a
frozen snapshot of the student matched via KL divergence at every
step~\citep{opsd2026,opd_survey_2026,rethinking_opd_2026}.
This combination substantially improves sample efficiency on agentic
benchmarks~\citep{tcod2026,adwin2026}, yet allocating that step-level
supervision across a trajectory remains an open and challenging problem
central to the whole recipe.

Existing agentic methods answer it by weighting each step by the
magnitude of its teacher-student divergence, on the single-turn intuition
that a large disagreement marks a mistake worth
correcting~\citep{sod2026,sdar2026,gear2026,stepopsd2026}, while a
parallel line keeps the same magnitude rule but acts on it indirectly,
routing the gap into the RL update or enriching the teacher with
privileged skills or
debate~\citep{rlsd2026,opid2026,skillsd2026,madopd2026}.
A separate family reweights the RL advantage itself with structure read
from the rollout, such as branching or graph-based credit scores, but
leaves the distillation loss
unweighted~\citep{appo2026,gepo2026}.
Across all of these lines the allocation rule remains the single-turn
one, in which the size of the teacher-student gap, alone or combined with
rollout structure on the RL side, decides where supervision goes.

\begin{figure}[t]
  \centering
  \includegraphics[width=\linewidth]{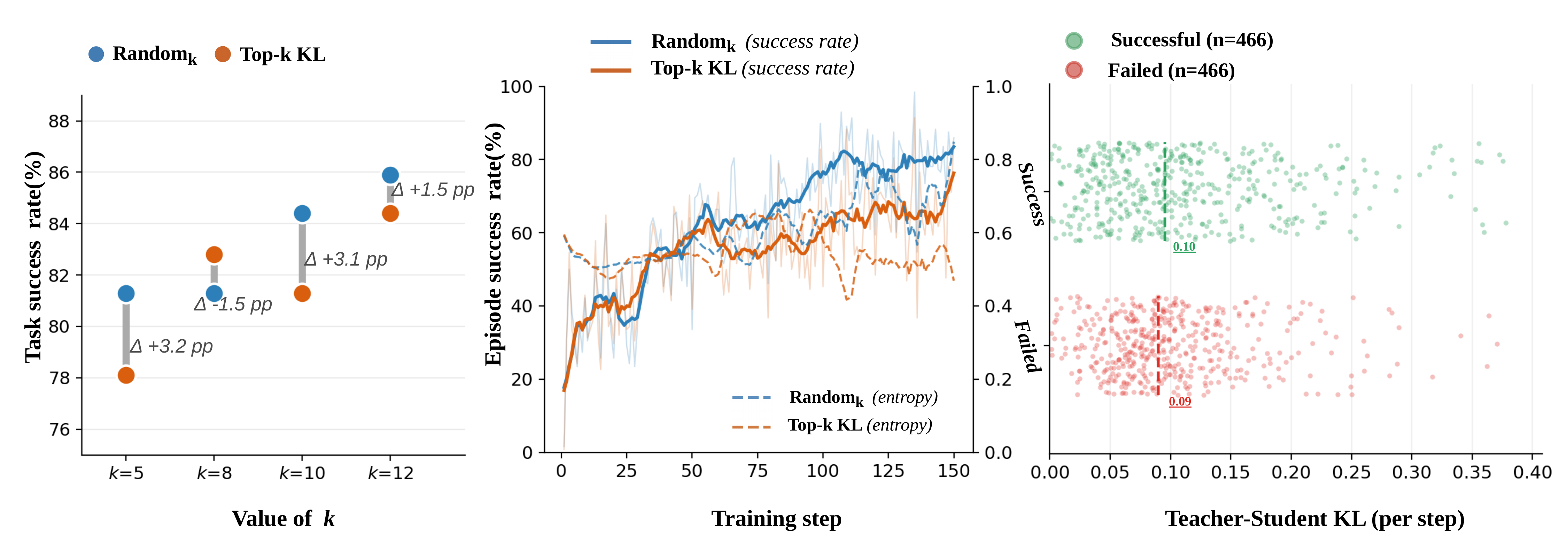}
  \caption{%
    \textbf{Motivating empirical study.}
    All runs distill only the $k$ highest-KL steps or $k$ randomly chosen
    steps per trajectory, $k \in \{5, 8, 10, 12\}$, on ALFWorld with a 7B
    Qwen2.5-Instruct agent under GRPO for 150 steps, all other conditions
    identical.
    (\textbf{Left}) Random selection stays ahead at every budget $k$, by
    $+1.5$ to $+3.2$ pp in  success rate, with the per-budget gaps
    $\Delta$ annotated above each pair.
    (\textbf{Middle}) At $k=5$ the KL-guided run trails in success rate
    (solid curves, left axis) throughout and ends at 76.3\% against 83.5\%
    for random, while its policy entropy (dashed curves, right axis)
    contracts as the random run's expands.
    (\textbf{Right}) The per-step KL of 466 successful and 466 failed steps
    from 50 rollouts overlaps by 84\% between the two outcomes, with
    medians (dashed) at 0.10 and 0.09, so success and failure are barely
    separable by divergence.%
  }
  \label{fig:motivation}
\end{figure}

As shown in Figure~\ref{fig:motivation}, we test that rule where decisions chain over many turns and find that it
fails exactly as the structure of agent trajectories predicts
, since distilling the $k$ highest-KL steps
brings no consistent benefit over distilling $k$ random steps, the
divergence-guided rule trails random selection throughout training while
policy entropy collapses rather than grows, and successful and failed
steps barely separate by divergence.
An early drift enters every later context both policies condition on,
leaving the teacher locally consistent with the drifted trajectory
precisely at the outcome-deciding steps, and interchangeable steps
register large divergences unrelated to the result, so divergence
magnitude measures how far the student has drifted rather than whether
that drift mattered.
Agent trajectories nevertheless record what divergence cannot see, namely
which steps enabled which later ones, readable from environment state
transitions rather than from either policy's predictions, and no prior
agentic on-policy distillation method has used this record to place the
teacher's supervision.

To this end, in this paper we focus on placing the step-level supervision
where the teacher's guidance remains usable along a multi-turn trajectory,
and propose \textbf{GraphOPD}, the first method to introduce graph-based
structural augmentation into on-policy distillation for improving agent
capabilities.
For each rollout, GraphOPD reads which earlier state changes a later action consumes into a directed dependency graph, converts its random-walk stationary distribution into a per-step structural credit score, and fuses that score with the divergence signal into a trajectory-relative mask that spends supervision where the teacher's guidance is usable, leaving the rest to the RL objective. The structural score marks the steps at which much of the trajectory's accumulated dependency structure is resolved, so it stays informative exactly where the teacher-student gap has been blinded by drift or inflated by interchangeable actions. Since both signals come from the rollout itself, the mechanism needs no extra model calls and enters training as a drop-in reweighting of the distillation loss.
We evaluate GraphOPD on ALFWorld, WebShop, and SearchQA across three
model scales against eleven baselines, where it shows competitive
performance throughout and improves over the strongest baseline by up to
\textbf{+5.8 pp}.
Analyses further confirm that both signals are independently necessary,
that the structural credit score tracks true causal impact far above
chance under an executed-replay audit, and that the same signal
transfers to out-of-domain tool-integrated reasoning.

The main contributions of this paper are summarized as follows:
\begin{itemize}
  \item We study the divergence-magnitude rule on multi-turn agent benchmarks and find that the teacher-student gap measures drift rather than outcome relevance, so distilling its highest-divergence steps brings no consistent benefit over random selection, motivating a drift-immune supervision signal for long-horizon agents.

  \item We are the first to bring graph-based structural augmentation into on-policy distillation for agents, reading a directed graph from environment state transitions to capture which steps the outcome turns on.

  \item We design GraphOPD, which fuses a random-walk structural credit
        score with the divergence signal into a distillation aptitude and
        converts it into a trajectory-relative mask that trains alongside
        the reinforcement-learning objective.

  \item Experiments across three benchmarks, three scales, and eleven
        baselines show gains up to +5.8 pp over the strongest baseline,
        with the structural credit score validated by executed-replay
        counterfactuals and out-of-domain transfer.
\end{itemize}

\section{Related Work}
\label{sec:related}

\paragraph{On-policy distillation for LLM agents.}
On-policy distillation augments reinforcement learning with dense
token-level supervision from a teacher evaluated on the student's own
rollout states~\citep{opd_survey_2026,rethinking_opd_2026}, and the methods
we compare against are either instances of it or bounds on what it can add.
Objective-level work stabilizes or generalizes the same recipe via target
reformulation, entropy-aware divergences, teacher-supported truncation, or
reward
extrapolation~\citep{jang2026veto,jin2026entropy,fu2026revisiting,yang2026learning}.
Among the on-policy methods, OPSD~\citep{opsd2026} distills the student's
own frozen snapshot over its rollouts and SDAR~\citep{sdar2026} applies the same recipe to
multi-turn agent rollouts, RLSD~\citep{rlsd2026} routes the teacher-student
gap into the RL objective as advantage shaping or update magnitude, and
OPID~\citep{opid2026} distills on-policy from a teacher augmented with
privileged skill context.
Skill-GRPO and
Skill-SD~\citep{skillsd2026} injecting external skill demonstrations or
conditioning self-distillation on them.
Each on-policy baseline allocates supervision by the magnitude of the
teacher-student divergence, which is the shared premise our empirical study
tests.

\paragraph{Step-level weighting and structural credit assignment.}
Beyond our comparison set, on-policy distillation has been extended along
three axes, controlling when to distill through temporal
curricula~\citep{tcod2026,adwin2026}, how strongly to supervise each step
from the divergence through gating, granularity adaptation, or advantage
routing~\citep{sod2026,gear2026,stepopsd2026}, and how reliable the teacher
is through debate, rollout-tested guidance, or bounded update
magnitudes~\citep{madopd2026,chen2026look,liu2026when,wang2026teach,sun2026step}.
A separate family reweights the RL advantage rather than the distillation
loss using structure read from the rollout: branching and graph-based
credit scores~\citep{appo2026,gepo2026,eglsca2026,masopt2026,mtgcausal2026},
dependency and memory
graphs~\citep{vimrag2026,llmexpl2026,stepcode2026},
verifier- or Shapley-based edge credit~\citep{shapleyflow2020},
milestone and rubric judgments~\citep{zhang2026trca}, and
bisimulation or transport distances to a reference
plan~\citep{wang2026bipace}.
Beyond agent RL, graphs underlie self-correcting knowledge-graph
planning~\citep{pog2024}, multi-modal entity
alignment~\citep{mmea2020,msnea2022}, temporal or de-correlated
recommendation~\citep{tgcl4sr2024,afdgcf2024}, and video social-graph
generation~\citep{wu2021linking}.
None of these represent, for the distillation loss, which steps the later
ones actually depend on, and that is what GraphOPD reads directly from the
trajectory.

\section{Method}
\label{sec:method}

GraphOPD fuses two per-step signals into a single distillation weight.
The first is a structural credit signal that identifies which steps in a
trajectory were critical to its outcome, derived from a lightweight
dependency graph built over the steps' entity-level relations and scored
with a random-walk stationary distribution.
The second is the divergence signal used by prior on-policy distillation
methods, measuring how far the student's behavior has drifted from the
teacher's.
GraphOPD combines the two into a trajectory-relative per-step weight that
spends distillation only where the teacher's guidance is both
trustworthy and useful, leaving the rest to the RL objective.
Figure~\ref{fig:framework} summarizes the full pipeline.

\begin{figure}[t]
  \centering
  \includegraphics[width=\linewidth]{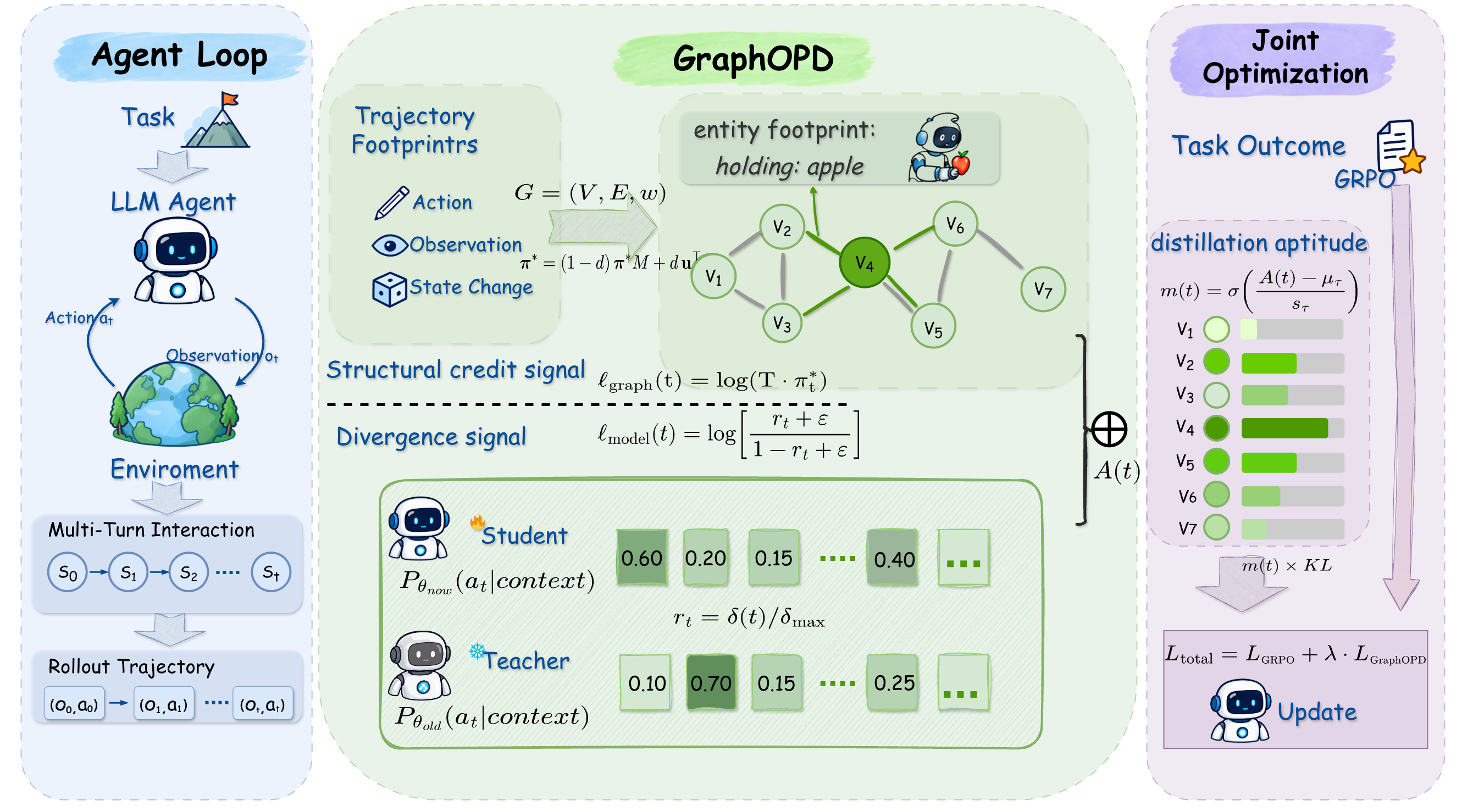}
  \caption{\textbf{GraphOPD overview.}
    From a single rollout, two per-step signals are computed in parallel, a structural credit signal from the Footprint-built dependency graph and its random-walk stationary distribution, and a divergence signal from the teacher-student gap normalized within the trajectory. The two are fused into a distillation aptitude $A(t)$, converted into a trajectory-relative mask $m(t)$, and the resulting masked distillation loss is added to the unweighted GRPO objective; every stage reads only the rollout itself and needs no extra model calls.}
  \label{fig:framework}
\end{figure}

\subsection{Problem Setup}
\label{sec:setup}

We consider an agent that interacts with an environment over $T$ steps and
produces a trajectory $\tau = (s_0, a_1, o_1, \ldots, a_T, o_T)$, where at
each step $t$, against environment state $s_t$, the agent emits action
$a_t$ and the environment returns observation $o_t$; the action and
observation text are what our Footprint extractor reads.
A terminal reward $R$, issued only once at the end of the episode, encodes
task success or failure, and no dense per-step reward is assumed.

The student policy $\pi_\theta$ is trained under GRPO,
augmented by a teacher signal at each step.
We take the teacher to be $\pi_{\theta_\text{old}}$, the student's snapshot
frozen at the start of each rollout batch and already used by GRPO for the
importance ratio.
Staying within one update of the current policy, it serves as an on-support
behavioral anchor rather than an external expert.
\subsection{Structural Credit Assignment}
\label{sec:causal_graph}

\paragraph{Assigning credit without a dense reward.}
With only a terminal reward, a trajectory offers no direct signal
for which of its steps actually mattered.
Agent actions constrain what is possible later, and this directed structure
is readable from action and observation text alone, without LLM calls,
privileged annotations, or rollout counterfactuals.
We use it as a proxy for step-level credit, made explicit through a
lightweight graph $G = (V, E, w)$, with one node $v_t \in V$ per step,
built from a rule-based \emph{State Transition Footprint} $\mathcal{F}(t)$
of entity-level state-change keys
\begin{equation}
\mathcal{F}(t) = \mathcal{F}_{\text{env}}(a_t,\,o_t)
  \;\subseteq\;
  \bigl\{\,\texttt{<type>:<entity>}\,\bigr\}.
\end{equation}
Full extractor specifications are in Appendix~\ref{app:footprint}.

\paragraph{Edge construction.}
A directed edge $v_i \to v_j$ ($i < j$) is added whenever step $i$'s state
change is explicitly referenced by step $j$'s action,
\begin{equation}
(v_i\!\to\!v_j)\in E
\;\iff\;
\exists\;\texttt{<type>:<entity>}\in\mathcal{F}(i)
\;\text{ s.t. }\;
\texttt{<entity>}\in\mathrm{substr}(a_j).
\label{eq:causal_edge}
\end{equation}
Two auxiliary types complete the graph.
\emph{Backtrack edges} of weight $0.5$ connect $j$ back to an earlier $i$
whenever $o_j = o_i$, marking a revisit, and \emph{sequential edges} of
weight $0.05$ connect consecutive steps lacking a dependency or backtrack
edge, preserving connectivity for the centrality computation.

\paragraph{Soft pruning by entity specificity.}
Substring matching can admit spurious dependency edges through
high-frequency entities that co-occur with many steps without carrying a
specific dependency.
Rather than a constant weight, we weight each dependency edge by the rarity
of the entity that induces it.
Let $df(e)$ be the number of steps in the trajectory whose action or
observation mentions entity $e$, and the dependency-edge weight is
\begin{equation}
w(v_i \to v_j) = \frac{1}{df(e)},
\label{eq:soft_prune}
\end{equation}
where $e$ is the matched entity, or the rarest one if several match.
An entity mentioned once or twice yields weight near $1$, and a ubiquitous
entity yields a small weight that is diluted once the stationary
distribution's row normalization is applied.
This soft pruning introduces no hyperparameter and leaves the backtrack and
sequential edges unchanged.

\paragraph{What the dependency graph encodes.}
The edges of $G$ record which earlier state changes a later action actually
consumes.
That is an enabling relation in the environment dynamics and a factual
property of the recorded rollout, so it does not inherit the noise affecting
the KL-based divergence signal in long-horizon settings.
We accordingly treat $\ell_\text{graph}$ as a \emph{structural credit
score}, measuring how much of the trajectory's later behavior runs
through a given step.
Whether that structural quantity tracks the steps the outcome really turns
on is an empirical question, and we settle it with an executed-replay
counterfactual audit in our experiments. Because a step's stationary mass is the rarity-weighted inflow of every
dependency chain its action consumes, intermediate milestones at which
several upstream state changes meet score above the neutral point just as
terminal actions do, and the trajectory-relative mask ranks steps within
their own rollout rather than by position.

\subsection{Structural Credit and Divergence Signals}
\label{sec:signals}

\paragraph{Structural credit signal \texorpdfstring{$\ell_\text{graph}$}{ℓ\_graph}.}
Counting only direct dependency edges misses transitive importance, so we
use a random-walk stationary distribution over the graph, which propagates
influence along edges rather than only counting them.
Let $M \in \mathbb{R}^{T \times T}$ be the row-normalized edge-weight
matrix, in which leaf nodes receive a uniform row $1/T$.
With damping $d = 0.15$ and uniform teleport
$\mathbf{u} = [1/T, \ldots, 1/T]^\top$, the stationary distribution
satisfies
\begin{equation}
\boldsymbol{\pi}^* = (1-d)\,\boldsymbol{\pi}^*M + d\,\mathbf{u}^\top.
\label{eq:stationary}
\end{equation}
Mass flows forward along the dependency edges, peaking at convergence
points where the preconditions assembled earlier are exercised and the
outcome is decided.
We convert it to log-odds relative to the uniform baseline,
\begin{equation}
\ell_\text{graph}(t) = \log(T \cdot \pi^*_t),
\label{eq:lgraph}
\end{equation}
which gives $\ell_\text{graph}(t) > 0$ above average and $0$ at uniform,
matching the neutral point of the divergence signal defined next.

\paragraph{Divergence signal \texorpdfstring{$\ell_\text{model}$}{ℓ\_model}.}
The divergence signal measures how much the student's step-$t$ behavior
differs from the teacher's.
Raw KL is dominated by global training-stage drift, so we use each step's
deviation relative to the trajectory, namely the mean absolute token
log-probability difference over the student's action tokens $S_t$,
\begin{equation}
\delta(t) = \frac{1}{|S_t|}\sum_{k\in S_t}
  \bigl|\log\pi_{\theta_\text{old}}(k\mid s_t)
       -\log\pi_\theta(k\mid s_t)\bigr|.
\label{eq:delta}
\end{equation}
We then convert the relative deviation
$r_t = \delta(t)/\delta_\text{max} \in [0,1]$ to a smoothed
within-trajectory log-odds score
\begin{equation}
\ell_\text{model}(t)
  = \log\!\left[\frac{r_t+\varepsilon}{1-r_t+\varepsilon}\right],
\quad
\delta_\text{max} = \max_{t'\in\tau}\delta(t'),
\label{eq:lmodel}
\end{equation}
with smoothing floor $\varepsilon = 0.1$.
The score is $0$ at mid-deviation, where $r_t=\tfrac12$, matching
$\ell_\text{graph}$'s neutral point, and the floor bounds it to
$\pm\log\frac{1+\varepsilon}{\varepsilon}\approx\pm2.4$, the same order as
$|\ell_\text{graph}|\le\log T$.
The two signals are therefore commensurate and either source can veto the
other, while the within-trajectory normalization keeps them comparable
across training stages.

\subsection{Distillation Aptitude, Masking, and Training Objective}
\label{sec:fusion}

\paragraph{Fusing the two signals into a distillation aptitude.}
Both signals already share a common footing, since each is zero at its own
neutral point and positive above it, and the $\varepsilon$-floor in
Eq.~\eqref{eq:lmodel} keeps $\ell_\text{model}$ on the same order as
$\ell_\text{graph}$'s $\pm\log T$ range established above, so neither can
unilaterally dominate the other.
We combine them by a weighted sum
\begin{equation}
A(t) = \alpha \cdot \ell_\text{model}(t)
      + (1-\alpha) \cdot \ell_\text{graph}(t),
\label{eq:aptitude}
\end{equation}
the \emph{Distillation Aptitude}.
$A(t)$ is high when both sources agree, meaning the student has
drifted \emph{and} the step is central, and it stays attenuated when the
two disagree.
$\alpha$ is selected by grid search on a validation split, with $\alpha=1$
recovering divergence-only weighting and $\alpha=0$ recovering
pure structural-credit weighting.

\paragraph{Trajectory-relative mask and training objective.}
The absolute scale of $A(t)$ is not stable across training, because
$\delta_\text{max}$ and the graph's stationary distribution both shift as
the policy changes, so a fixed threshold on $A(t)$ would need constant
re-tuning.
We instead standardize $A(t)$ within its own trajectory before masking.
With $\mu_\tau$ and $s_\tau$ the mean and standard deviation of
$\{A(t)\}_{t=1}^{T}$ over $\tau$,
\begin{equation}
m(t) = \sigma\!\left(\frac{A(t) - \mu_\tau}{s_\tau}\right),
\label{eq:mask}
\end{equation}
so $m(t) > \tfrac12$ exactly when step $t$'s aptitude exceeds its own
trajectory's average, and this relative threshold stays meaningful at every
stage of training because it is defined against $\tau$ itself rather than
against a fixed scale.
Every trajectory thus spends its own distillation budget on its own most
decisive steps, independent of trajectory length, dependency density, or how
far training has progressed.
Design and numerical details, including the degenerate $s_\tau=0$ case, are
deferred to Appendix~\ref{app:theory}.

The distillation loss aggregates the masked forward KL over each trajectory
step, normalized by the total mask weight,
\begin{equation}
\mathcal{L}_{\text{GraphOPD}}
  = \frac{1}{\sum_t m(t)}
    \sum_{t=1}^{T} m(t) \cdot
    \mathrm{KL}\!\bigl(\pi_{\theta_\text{old}}(\cdot \mid s_t)
                       \,\|\, \pi_\theta(\cdot \mid s_t)\bigr).
\label{eq:graphopd_loss}
\end{equation}
This loss is combined with the GRPO policy gradient objective as
\begin{equation}
\mathcal{L}_{\text{total}} = \mathcal{L}_{\text{GRPO}}
  + \lambda \cdot \mathcal{L}_{\text{GraphOPD}},
\label{eq:total_loss}
\end{equation}
where $\lambda = 0.05$ sets the weight of the distillation term relative
to the RL objective and is held fixed across all benchmarks and scales.
GraphOPD is implemented as a drop-in wrapper around the training loop,
and Algorithm~\ref{alg:graphopd} summarizes the full pipeline.

\section{Experiments}
\label{sec:experiments}

\subsection{Experimental Setup}
\label{sec:exp_setup}

We evaluate on three agentic benchmarks, ALFWorld~\citep{shridhar2021alfworld},
WebShop~\citep{yao2022webshop}, and SearchQA, a suite of seven open-domain
question-answering datasets~\citep{kocisky2018narrativeqa,joshi2017triviaqa,
mallen2023popqa,yang2018hotpotqa,ho2020wikimultihopqa,trivedi2022musique,
press2023bamboogle}, scored by episode success rate, by success rate
together with completion score, and by exact-match accuracy respectively;
the per-benchmark average columns are the headline numbers quoted in the
text.
All methods share the same GRPO~\citep{grpo2024} backbone implemented in
\texttt{veRL}~\citep{sheng2024hybridflow} and are trained at three model
scales, Qwen3-1.7B~\citep{qwen3_2025}, Qwen2.5-3B-Instruct and
Qwen2.5-7B-Instruct~\citep{qwen2024qwen25}, against eleven baselines;
baseline families, prompts, training configurations, and the run-variance
table are detailed in Appendix~\ref{app:setup}.

\subsection{Main Results}
\label{sec:main_results}

GraphOPD turns the dependency structure of a trajectory into a training
signal that existing on-policy distillation methods do not use, and
Table~\ref{tab:main_results} shows this signal translating directly into
consistent gains.
The advantage is most direct on the two benchmarks that require sustained
manipulation of an environment over many turns, where GraphOPD attains the
best average success rate on ALFWorld and the best success rate on WebShop
at every model scale, ahead of all eleven baselines, peaking at 91.9\%
ALFWorld success and 83.6\% WebShop success at 7B.
The margin over GEPO, the closest prior baseline, which reweights the
GRPO advantage with a graph-based credit score rather than adding a
distillation loss, is positive at every scale, reaching +5.2 to +7.9 pp
on ALFWorld and +3.6 to +5.5 pp on WebShop success rate.
GraphOPD also outperforms the skill-augmented baselines that receive
privileged demonstrations or validated skill libraries, such as
Skill-GRPO and GRPO+OPSD, on both benchmarks at every scale, despite
reading its own signal for free from the rollout rather than from
external supervision.
The same structural signal also carries over to SearchQA, a
retrieval-and-reasoning setting with a substantially different action
space, where GraphOPD leads at 3B and 1.7B with 46.1 and 42.9 average
accuracy, ahead of every direct-comparison baseline, and stays
competitive at 7B.
The margin over SDAR is positive on both benchmarks at every scale, at
\textbf{+12.0 pp} on ALFWorld and \textbf{+13.5 pp} on WebShop at 1.7B,
\textbf{+6.8 pp} and \textbf{+9.6 pp} at 3B, and \textbf{+7.5 pp} and
\textbf{+1.6 pp} at 7B.
The run-to-run standard deviations behind these means, given in
Table~\ref{tab:run_std}, confirm that the comparisons sit within
reasonable fluctuation: GraphOPD's own spread never exceeds 1.6 pp and
stays below 0.6 pp in seven of the nine headline cells, well under half
the largest baseline spreads of about 5 pp.

\begin{table}[t]
  \centering
  \caption{%
    \textbf{Main results.}
    Success rate (ALFWorld), accuracy (SearchQA), and success rate (WebShop) across three model scales.\textbf{Bold}  and \underline{underline} marks the best and the second best result per column within each model block.
  }
  \label{tab:main_results}
  \renewcommand{\arraystretch}{0.80}
  \setlength{\tabcolsep}{4pt}
  \resizebox{\textwidth}{!}{%
  \footnotesize
  \begin{tabular}{@{}l ccccccc cccccccc cc@{}}
    \toprule
    & \multicolumn{7}{c}{\textbf{ALFWorld}} & \multicolumn{8}{c}{\textbf{SearchQA}} & \multicolumn{2}{c}{\textbf{WebShop}} \\
    \cmidrule(lr){2-8} \cmidrule(lr){9-16} \cmidrule(lr){17-18}
    \textbf{Method} & Pick & Look & Clean & Heat & Cool & Pick2 & \cellcolor{yellow!15}Avg
      & NarQA & Triv & Pop & Hotp & 2Wk & MuS & Bam & \cellcolor{yellow!15}Avg
      & Score & \cellcolor{blue!10}Succ. \\
    \midrule
    \multicolumn{18}{l}{\textit{Qwen2.5-3B-Instruct}} \\
    \midrule
    Vanilla & 44.4 & 11.1 & 6.2 & 15.4 & 28.6 & 12.5 & \cellcolor{yellow!15}21.9
      & 24.6 & 48.1 & 31.0 & 26.3 & 25.3 & 7.2 & 59.7 & \cellcolor{yellow!15}31.7
      & 6.7 & \cellcolor{blue!10}0.8 \\
    GRPO & 91.2 & 62.5 & \underline{96.2} & 61.9 & 65.0 & 47.4 & \cellcolor{yellow!15}75.0
      & 39.3 & 60.6 & 41.1 & 37.4 & 34.6 & 15.4 & 26.4 & \cellcolor{yellow!15}36.4
      & 79.8 & \cellcolor{blue!10}63.3 \\
    Skill-Prompt & 51.7 & 66.7 & 48.4 & 0.0 & 4.3 & 10.0 & \cellcolor{yellow!15}28.9
      & 23.7 & 46.2 & 30.6 & 24.4 & 22.1 & 7.5 & 12.5 & \cellcolor{yellow!15}23.9
      & 0.2 & \cellcolor{blue!10}0.8 \\
    OPSD & 48.8 & 41.7 & 16.7 & 0.0 & 15.8 & 16.7 & \cellcolor{yellow!15}28.1
      & 0.1 & 0.1 & 0.1 & 0.0 & 0.0 & 0.0 & 0.0 & \cellcolor{yellow!15}0.0
      & 11.3 & \cellcolor{blue!10}3.1 \\
    Skill-GRPO & \underline{94.3} & 57.1 & \textbf{100.0} & 66.7 & 73.1 & 57.1 & \cellcolor{yellow!15}80.5
      & 44.3 & 59.6 & 44.3 & 39.0 & 36.1 & 14.5 & 14.9 & \cellcolor{yellow!15}36.1
      & 76.3 & \cellcolor{blue!10}66.4 \\
    GRPO+OPSD & \textbf{100.0} & \underline{82.4} & 85.7 & \underline{75.0} & 70.0 & 60.0 & \cellcolor{yellow!15}81.2
      & \underline{44.9} & 61.2 & \textbf{45.2} & 40.4 & 38.5 & 16.0 & \underline{66.1} & \cellcolor{yellow!15}\underline{44.6}
      & 77.8 & \cellcolor{blue!10}66.4 \\
    Skill-SD & 88.2 & 50.0 & \underline{96.2} & 52.4 & 65.0 & 57.9 & \cellcolor{yellow!15}73.4
      & 44.4 & 60.4 & 44.0 & 39.5 & \underline{40.4} & 15.4 & 64.9 & \cellcolor{yellow!15}44.1
      & 75.9 & \cellcolor{blue!10}64.0 \\
    RLSD & 87.9 & 75.0 & 90.9 & \underline{75.0} & 73.1 & 68.4 & \cellcolor{yellow!15}79.7
      & 41.5 & 58.6 & 42.3 & 40.4 & 40.2 & 16.8 & \textbf{66.9} & \cellcolor{yellow!15}43.8
      & 84.4 & \cellcolor{blue!10}66.4 \\
    SDAR & 85.2 & 69.8 & 88.6 & \textbf{85.6} & \underline{80.4} & \underline{74.1} & \cellcolor{yellow!15}81.8
      & 43.7 & 59.2 & 43.9 & 38.0 & 36.4 & 14.9 & 63.7 & \cellcolor{yellow!15}43.4
      & 83.8 & \cellcolor{blue!10}69.3 \\
    OPID & 94.2 & \textbf{100.0} & 90.0 & 66.7 & 74.1 & 63.1 & \cellcolor{yellow!15}\underline{82.8}
      & \textbf{45.1} & \underline{61.4} & 43.4 & \underline{40.5} & 38.9 & 17.2 & 60.3 & \cellcolor{yellow!15}\underline{44.6}
      & 83.7 & \cellcolor{blue!10}74.5 \\
    GEPO & 90.7 & 76.4 & 89.5 & 64.8 & 78.2 & 72.0 & \cellcolor{yellow!15}80.7
      & 43.7 & \textbf{62.1} & 44.4 & 40.2 & 37.5 & \textbf{17.9} & 65.0 & \cellcolor{yellow!15}\underline{44.6}
      & \underline{88.1} & \cellcolor{blue!10}\underline{75.3} \\
    \midrule
    \rowcolor{gray!8}
    \textbf{Ours} & 90.5 & 73.5 & 92.9 & \textbf{85.6} & \textbf{89.6} & \textbf{88.3} & \cellcolor{yellow!15}\textbf{88.6}
      & 43.8 & 60.3 & \underline{45.0} & \textbf{41.0} & \textbf{41.8} & \underline{17.4} & 64.6 & \cellcolor{yellow!15}\textbf{46.1}
      & \textbf{90.7} & \cellcolor{blue!10}\textbf{78.9} \\
    \midrule
    \multicolumn{18}{l}{\textit{Qwen2.5-7B-Instruct}} \\
    \midrule
    Vanilla & 36.1 & 22.2 & 3.1 & 0.0 & 0.0 & 0.0 & \cellcolor{yellow!15}12.5
      & 25.2 & 50.8 & 29.5 & 29.0 & 29.0 & 10.4 & 63.7 & \cellcolor{yellow!15}33.9
      & 5.9 & \cellcolor{blue!10}1.6 \\
    GRPO & 91.2 & \underline{87.5} & 96.2 & 81.0 & 65.0 & 57.9 & \cellcolor{yellow!15}81.2
      & 45.1 & 63.7 & 44.0 & 43.6 & 43.2 & 16.8 & 37.6 & \cellcolor{yellow!15}42.0
      & 80.9 & \cellcolor{blue!10}72.6 \\
    Skill-Prompt & 51.7 & 50.0 & 32.3 & 5.3 & 4.3 & 0.0 & \cellcolor{yellow!15}23.4
      & 30.9 & 52.1 & 32.7 & 32.7 & 27.9 & 12.7 & 66.1 & \cellcolor{yellow!15}36.4
      & 1.7 & \cellcolor{blue!10}0.8 \\
    OPSD & 50.0 & 60.0 & 22.7 & 21.4 & 17.6 & 9.5 & \cellcolor{yellow!15}32.8
      & 8.8 & 8.6 & 17.5 & 2.5 & 4.2 & 0.5 & 1.2 & \cellcolor{yellow!15}6.2
      & 4.5 & \cellcolor{blue!10}2.3 \\
    Skill-GRPO & \textbf{100.0} & 83.3 & 96.4 & 83.3 & 75.0 & 78.9 & \cellcolor{yellow!15}88.3
      & 44.8 & 63.0 & 45.1 & 43.7 & 43.7 & \underline{20.5} & \underline{71.4} & \cellcolor{yellow!15}47.5
      & 87.0 & \cellcolor{blue!10}81.2 \\
    GRPO+OPSD & 91.4 & 61.5 & \textbf{100.0} & 87.5 & 76.5 & 52.2 & \cellcolor{yellow!15}80.4
      & \textbf{47.3} & \textbf{64.5} & \underline{46.9} & 43.8 & 39.3 & 18.0 & 69.4 & \cellcolor{yellow!15}47.0
      & 86.8 & \cellcolor{blue!10}76.5 \\
    Skill-SD & 93.9 & \textbf{93.8} & 90.9 & \textbf{100.0} & 69.2 & 68.4 & \cellcolor{yellow!15}85.1
      & 47.1 & \textbf{64.5} & \textbf{47.8} & 44.2 & 42.1 & 20.2 & 69.0 & \cellcolor{yellow!15}47.8
      & 86.1 & \cellcolor{blue!10}76.5 \\
    RLSD & \textbf{100.0} & \underline{87.5} & 92.3 & 58.8 & 80.0 & 65.2 & \cellcolor{yellow!15}82.0
      & 46.8 & 63.0 & 44.4 & \textbf{45.5} & \textbf{48.9} & \textbf{21.5} & \textbf{73.0} & \cellcolor{yellow!15}\textbf{49.0}
      & 87.4 & \cellcolor{blue!10}77.3 \\
    SDAR & 91.1 & 78.9 & \underline{97.2} & 82.8 & 75.5 & 76.8 & \cellcolor{yellow!15}84.4
      & 46.4 & 63.0 & 44.7 & 43.2 & \underline{44.5} & 19.4 & 64.1 & \cellcolor{yellow!15}47.5
      & \underline{88.0} & \cellcolor{blue!10}\underline{82.0} \\
    OPID & 97.5 & 80.0 & 96.8 & 93.6 & \underline{86.9} & 78.5 & \cellcolor{yellow!15}\underline{89.1}
      & \underline{47.2} & \underline{64.2} & 46.1 & \underline{44.4} & 43.0 & 20.1 & 67.8 & \cellcolor{yellow!15}\underline{48.2}
      & 87.0 & \cellcolor{blue!10}80.0 \\
    GEPO & 90.7 & 80.1 & 96.3 & 77.3 & 83.4 & \underline{79.2} & \cellcolor{yellow!15}85.9
      & 43.9 & 61.6 & 44.1 & 40.8 & 41.3 & 17.5 & 65.2 & \cellcolor{yellow!15}45.5
      & 87.1 & \cellcolor{blue!10}80.0 \\
    \midrule
    \rowcolor{gray!8}
    \textbf{Ours} & \underline{97.8} & 70.6 & 94.5 & \underline{97.2} & \textbf{92.5} & \textbf{88.9} & \cellcolor{yellow!15}\textbf{91.9}
      & 46.4 & 63.3 & 45.2 & 43.2 & 42.9 & 19.2 & 64.0 & \cellcolor{yellow!15}47.3
      & \textbf{90.4} & \cellcolor{blue!10}\textbf{83.6} \\
    \midrule
    \multicolumn{18}{l}{\textit{Qwen3-1.7B}} \\
    \midrule
    Vanilla & 25.0 & 22.2 & 3.1 & 0.0 & 21.4 & 4.2 & \cellcolor{yellow!15}12.5
      & 29.4 & 46.9 & 37.0 & 23.5 & 19.6 & 6.4 & 10.5 & \cellcolor{yellow!15}24.8
      & 46.5 & \cellcolor{blue!10}4.7 \\
    GRPO & 71.1 & 41.7 & 36.4 & 40.0 & 31.8 & 31.6 & \cellcolor{yellow!15}46.1
      & 40.0 & 58.9 & 43.5 & 35.4 & 30.3 & 12.0 & \underline{65.7} & \cellcolor{yellow!15}40.8
      & 67.3 & \cellcolor{blue!10}38.3 \\
    Skill-Prompt & 10.3 & 50.0 & 16.1 & 0.0 & 0.0 & 5.0 & \cellcolor{yellow!15}9.4
      & 29.4 & 46.5 & 36.2 & 22.9 & 20.8 & 4.3 & 10.1 & \cellcolor{yellow!15}24.3
      & 23.0 & \cellcolor{blue!10}2.3 \\
    OPSD & 26.3 & 33.3 & 9.1 & 0.0 & 4.5 & 5.3 & \cellcolor{yellow!15}14.1
      & 4.2 & 8.3 & 4.6 & 6.6 & 15.3 & 0.7 & 1.2 & \cellcolor{yellow!15}5.8
      & 47.4 & \cellcolor{blue!10}9.3 \\
    Skill-GRPO & 31.4 & 42.9 & 51.9 & 8.3 & 11.5 & 7.1 & \cellcolor{yellow!15}28.1
      & 38.0 & 58.4 & 43.9 & 36.3 & 29.0 & 12.5 & \textbf{66.9} & \cellcolor{yellow!15}40.7
      & 80.4 & \cellcolor{blue!10}50.0 \\
    GRPO+OPSD & 38.2 & 50.0 & 30.8 & 28.6 & 30.0 & 21.1 & \cellcolor{yellow!15}32.0
      & \underline{40.7} & 58.9 & \underline{45.0} & \underline{37.0} & 34.6 & 13.3 & \underline{65.7} & \cellcolor{yellow!15}\underline{42.2}
      & 70.7 & \cellcolor{blue!10}38.3 \\
    Skill-SD & 52.9 & 37.5 & 69.2 & 42.9 & \underline{60.0} & 36.8 & \cellcolor{yellow!15}52.3
      & 39.1 & 57.5 & \textbf{45.4} & 34.8 & 34.1 & 10.7 & 64.1 & \cellcolor{yellow!15}40.8
      & 81.8 & \cellcolor{blue!10}53.9 \\
    RLSD & 50.0 & 37.5 & 61.5 & 19.0 & 50.0 & 21.1 & \cellcolor{yellow!15}42.2
      & 38.6 & 57.3 & 43.0 & 34.5 & 34.1 & 11.5 & 65.3 & \cellcolor{yellow!15}40.6
      & 74.0 & \cellcolor{blue!10}50.8 \\
    SDAR & \textbf{77.0} & 48.2 & 61.4 & 34.1 & 32.8 & 31.3 & \cellcolor{yellow!15}51.3
      & 38.7 & 57.8 & 41.8 & 35.2 & 33.9 & 12.9 & 62.1 & \cellcolor{yellow!15}40.7
      & 76.7 & \cellcolor{blue!10}59.9 \\
    OPID & \underline{75.2} & \textbf{59.0} & 61.9 & \underline{62.5} & 50.9 & \textbf{37.9} & \cellcolor{yellow!15}\underline{58.2}
      & 37.3 & 57.6 & 40.4 & 34.7 & 31.5 & 13.0 & 62.9 & \cellcolor{yellow!15}39.6
      & 76.7 & \cellcolor{blue!10}66.4 \\
    GEPO & 72.9 & 54.0 & \textbf{71.6} & \textbf{65.9} & 45.9 & 35.4 & \cellcolor{yellow!15}58.1
      & 39.4 & \underline{59.0} & 42.0 & 36.3 & \underline{34.9} & \underline{13.9} & 53.5 & \cellcolor{yellow!15}40.1
      & \underline{82.6} & \cellcolor{blue!10}\underline{67.9} \\
    \midrule
    \rowcolor{gray!8}
    \textbf{Ours} & 74.5 & \underline{56.7} & \underline{71.1} & 61.1 & \textbf{70.9} & \underline{37.1} & \cellcolor{yellow!15}\textbf{63.3}
      & \textbf{42.8} & \textbf{59.6} & 40.3 & \textbf{39.3} & \textbf{35.1} & \textbf{16.3} & 59.3 & \cellcolor{yellow!15}\textbf{42.9}
      & \textbf{85.7} & \cellcolor{blue!10}\textbf{73.4} \\
    \bottomrule
  \end{tabular}%
  }
\end{table}

\par\medskip
\noindent
\begin{minipage}{\linewidth}

  \begin{minipage}[t]{0.57\linewidth}
    \vspace{0pt}
    \centering
    \captionof{table}{%
      ALFWorld and WebShop success rate in percent for
      Qwen2.5-7B-Instruct across the signal-mixing
      coefficient $\alpha$.
    }
    \label{tab:alpha_sweep_ws}

    \small
    \setlength{\tabcolsep}{3pt}
    \begin{tabular*}{\linewidth}{@{\extracolsep{\fill}}lccccc@{}}
      \toprule
      $\alpha$ & 0 & 0.2 & 0.5 & 0.7 & 1.0 \\
      \midrule
      ALFWorld
        & 89.1 & 90.6 & \textbf{91.9} & 87.5 & 87.5 \\
      WebShop
        & 81.3 & 79.7 & \textbf{83.6} & 76.6 & 79.6 \\
      \bottomrule
    \end{tabular*}
  \end{minipage}\hfill%
  \begin{minipage}[t]{0.40\linewidth}
    \vspace{0pt}
    \centering
    \captionof{table}{%
      \textbf{Credit-content audit on ALFWorld at 7B.}
      Metrics defined in the text.
    }
    \label{tab:credit_dose}

    \small
    \setlength{\tabcolsep}{3pt}
    \begin{tabular*}{\linewidth}{@{\extracolsep{\fill}}lcc@{}}
      \toprule
      Signal & $\rho$ & Hit@10\% \\
      \midrule
      $\ell_{\text{graph}}$ & \textbf{0.58} & \textbf{0.48} \\
      $\ell_{\text{model}}$ & 0.34 & 0.41 \\
      Random               & $-0.06$ & 0.13 \\
      \bottomrule
    \end{tabular*}
  \end{minipage}

\end{minipage}
\par\medskip

\subsection{Ablation Study}
\label{sec:ablation}

We ablate four design choices to isolate each component's contribution, and
add a shuffled-centrality control to test whether the gain comes from the
structural credit signal's content or from the mere act of non-uniform
weighting.
Table~\ref{tab:ablation} reports results on WebShop, SearchQA, and
ALFWorld.
Divergence-only, which drops the structural credit signal, costs 3.9 to
10.5 pp, so the structural credit signal carries substantial information
about which steps matter.
Structural-only, which drops the divergence signal, is also necessary at
2.3 to 6.5 pp, and the two signals therefore act as complements rather than
substitutes.
Soft-Pruning, which replaces the entity-specificity dependency-edge
weighting with a constant weight, costs 0.8 to 5.5 pp, less than dropping
either signal outright at every one of the nine benchmark-scale cells.
Shuffled Centrality keeps the mask's sparsity and weight distribution but
randomizes which step receives which weight.
At 7B, and on WebShop and ALFWorld at 3B, it costs \emph{more} than
Divergence-only by up to 1.6 pp, so misassigning the structural credit
scores to the wrong steps hurts more than having no structural signal at
all.
At 1.7B, and on SearchQA at 3B, the pattern reverses and Shuffled costs no
more than Divergence-only, most sharply on SearchQA at 1.7B, where the two
cost 2.8 pp and 10.5 pp respectively.
At the smallest scale, then, the redistribution of supervision weight itself
carries part of the structural credit signal's benefit.

\begin{table}[t]
  \centering
  \caption{%
    \textbf{Ablation study.}
    Success rate (WebShop), accuracy (SearchQA), and success rate
    (ALFWorld).
    Variant definitions are given in the text above, and the GraphOPD row
    reproduces the three-run means of the full method, while each ablation
    variant is a single-run point estimate.
    \textbf{Bold} marks the best result per column within each model
    block.
  }
  \label{tab:ablation}
  \renewcommand{\arraystretch}{0.88}
  \setlength{\tabcolsep}{4pt}
  \resizebox{\linewidth}{!}{%
  \footnotesize
  \begin{tabular}{@{}l ccc ccc ccc@{}}
    \toprule
    & \multicolumn{3}{c}{\textbf{Qwen3-1.7B}}
    & \multicolumn{3}{c}{\textbf{Qwen2.5-3B}}
    & \multicolumn{3}{c}{\textbf{Qwen2.5-7B}} \\
    \cmidrule(lr){2-4} \cmidrule(lr){5-7} \cmidrule(lr){8-10}
    \textbf{Method}
      & WebShop
      & SearchQA
      & ALFWorld
      & WebShop
      & SearchQA
      & ALFWorld
      & WebShop
      & SearchQA
      & ALFWorld \\
    \midrule
    \textbf{GraphOPD}
      & \textbf{73.4} & \textbf{42.9} & \textbf{63.3}
      & \textbf{78.9} & \textbf{46.1} & \textbf{88.6}
      & \textbf{83.6} & \textbf{47.3} & \textbf{91.9} \\
    \quad Soft-Pruning
      & 71.1 & 37.4 & 62.5
      & 77.3 & 44.2 & 85.9
      & 82.0 & 45.3 & 90.6 \\
    \quad Divergence-only
      & 64.8 & 32.4 & 59.4
      & 74.2 & 40.8 & 84.4
      & 79.6 & 42.3 & 87.5 \\
    \quad Structural-only
      & 69.5 & 36.4 & 60.2
      & 75.0 & 43.3 & 84.4
      & 81.3 & 43.6 & 89.1 \\
    \quad Shuffled Centrality
      & 68.8 & 40.1 & 59.4
      & 72.7 & 41.2 & 82.8
      & 78.1 & 41.9 & 85.9 \\
    \bottomrule
  \end{tabular}%
  }
\end{table}

\paragraph{Hyperparameter sensitivity.}
\label{sec:hparam_search}
We select the mixing coefficient $\alpha$ by grid search on the ALFWorld
validation split, fixing the distillation loss coefficient $\lambda=0.05$,
and the selected configuration of $\alpha=0.5$ with $\lambda=0.05$ is
reported alongside the $\alpha$ sweep.
The sweep confirms that the gain comes from fusing both signals, since on
both ALFWorld and WebShop in Table~\ref{tab:alpha_sweep_ws} every other
$\alpha$ underperforms the selected $\alpha=0.5$, with a consistent
asymmetric cost at the pure-signal endpoints of about 4.0 to 4.4 pp for
dropping the structural credit signal at $\alpha=1$ against 2.3 to 2.8 pp
for dropping the divergence signal at $\alpha=0$.
The structural credit signal is therefore the larger single-signal
contributor.

\subsection{Analysis}
\label{sec:analysis}

\begin{figure}[h!]
  \centering
  \includegraphics[width=\linewidth]{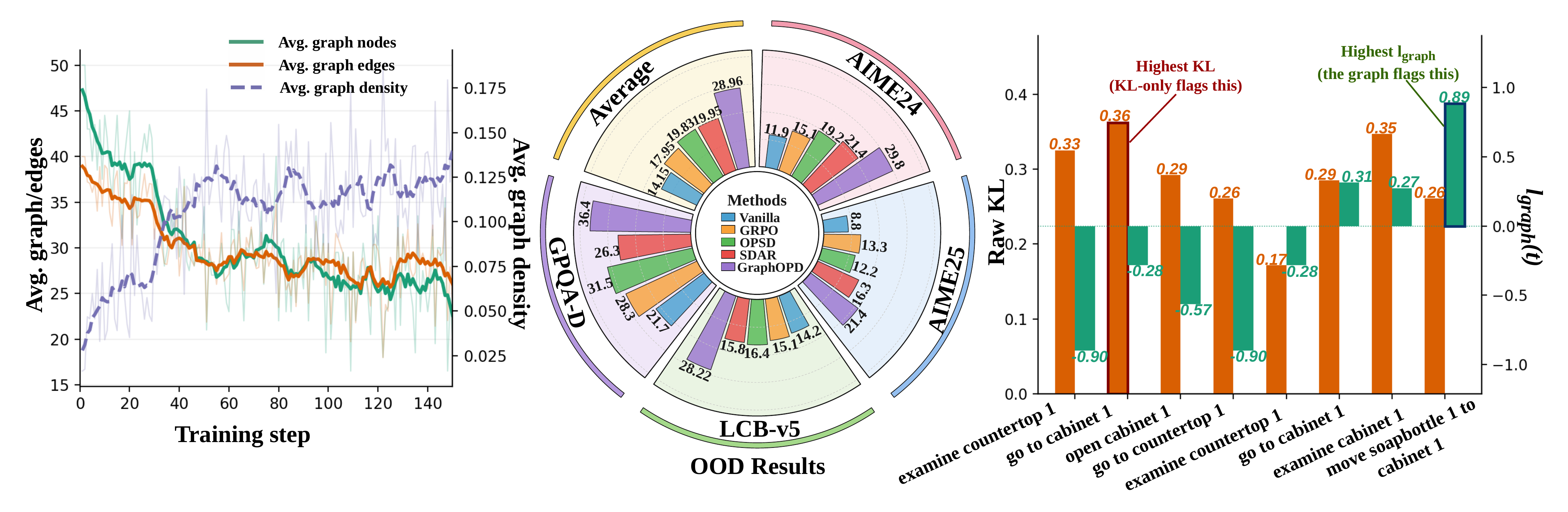}
  \caption{%
    \textbf{Analysis.}
    (\textbf{Left}) Average graph nodes, equivalently trajectory length,
    average edges, and average density of the dependency graphs built on
    ALFWorld at 7B over the 150 training steps; nodes and edges fall from
    about 47 and 39 to a plateau of 25 to 30 while density rises from about
    0.03 to between 0.10 and 0.14, so trajectories grow shorter and denser.
    (\textbf{Middle}) Accuracy radar on the four out-of-domain
    tool-integrated-reasoning benchmarks, AIME24 and AIME25 under avg@32
    and LiveCodeBench-v5 and GPQA-Diamond under avg@8, plus their average,
    for GraphOPD against Vanilla, GRPO, OPSD, and SDAR; the GraphOPD
    polygon dominates every axis, averaging 29.0\% against 20.0\% for SDAR.
    Figure~\ref{fig:analysis} (right) makes the mechanism concrete on a fragment of one real successful ALFWorld trajectory, where raw KL peaks at a cabinet revisit that never decides the outcome while $\ell_\text{graph}$ peaks at the decisive object-placement step within this fragment.
  }
  \label{fig:analysis}
\end{figure}

\paragraph{Dependency graph structure over training.}
Figure~\ref{fig:analysis} (left) tracks how GraphOPD reshapes the dependency
graph over training, testing whether centrality-weighted supervision pushes
the policy toward shorter, more decisive trajectories rather than merely
re-weighting its existing behavior.
It does, since the average number of graph nodes, equivalently trajectory
length, and the average number of edges both fall over the first 150 steps,
from about 47 and 39 at initialization to a plateau around 25 to 30, while
graph density rises correspondingly from about 0.03 to between 0.10 and
0.14, consistent with the policy reaching the goal through fewer, more
decisive actions rather than the same actions simply being scored
differently.

\paragraph{Credit-content audit.}
Table~\ref{tab:credit_dose} validates the structural credit signal.
The ablations show that the structural credit signal improves outcomes,
and the audit tests directly whether it tracks the steps that decide the
outcome, via executed replay~\citep{zhang2026credit}.
For 30 trajectories of the trained 7B policy we sample 10 decision steps
per trajectory, re-sample 4 alternative actions at each, roll forward to
the horizon, and take the induced change in success as ground-truth
impact.
For each signal we report the Spearman correlation $\rho$ with this impact
and Hit@10\%, the fraction of its top-10\% predicted steps that are truly
top-10\% impactful, for which chance is $0.10$.
$\ell_\text{graph}$ scores $\rho=0.58$ with Hit@10\% $=0.48$, clearly above
both the random baseline at $\rho=-0.06$ and Hit@10\% $=0.13$ and
$\ell_\text{model}$ at $\rho=0.34$ and Hit@10\% $=0.41$.
The structural credit signal thus identifies outcome-relevant steps more
reliably than the divergence gap alone, and the Shuffled Centrality control
isolates this same gain from mere non-uniform weighting.

\paragraph{Generalization and case study.}
Figure~\ref{fig:analysis} (middle) tests whether the dependency-graph
mechanism transfers beyond the three agentic benchmarks to a different
domain and tool, on AIME24, AIME25, LiveCodeBench-v5, and GPQA-Diamond run
as tool-integrated reasoning in mathematics, science, and code against
Vanilla, GRPO, OPSD, and SDAR under the same backbone and
Qwen2.5-7B-Instruct, and scored by avg@32 on the two AIME sets and avg@8
on the other two.
GraphOPD attains the best accuracy on all four, averaging 29.0\% against
20.0\% for SDAR and 19.8\% for OPSD, a +9.0 pp average gain over SDAR that
reaches +12.4 pp on LiveCodeBench-v5; its lead is the one pattern holding
on every benchmark, showing that the structural signal transfers where
weighting by distributional deviation alone does not.
Prompts, the Footprint extractor, and the full result table are in
Appendix~\ref{app:ood}.
Figure~\ref{fig:analysis} (right) makes the mechanism concrete on a
fragment of one real successful ALFWorld trajectory, where raw KL peaks at
a cabinet revisit that never decides the outcome while $\ell_\text{graph}$
peaks at the decisive object-placement step within this fragment.

\section{Conclusion}
\label{sec:conclusion}
We presented \textbf{GraphOPD}, which reads a per-step
structural credit score from the trajectory and fuses it with the
divergence signal to decide where a teacher's guidance is usable,
distilling only each rollout's above-average-aptitude steps.
Across three model scales, eleven baselines, and three agentic benchmarks
spanning embodied control, web navigation, and retrieval-augmented
reasoning, GraphOPD attains the best result on ALFWorld and WebShop at
every scale, with a margin over the strongest baseline that reaches
\textbf{+5.8 pp} on ALFWorld and \textbf{+5.5 pp} on WebShop, and the same
structural signal proves effective on SearchQA as well, extending the
approach to a substantially different action space.
Ablations confirm that both signals are independently necessary, with the
structural credit signal the larger contributor.
Using an LLM to extract the Footprint directly from raw trajectories,
rather than through hand-specified parsing rules, is a natural next step
toward applying GraphOPD to a still broader range of agentic domains.

\newpage
\section*{AI use statement}
\label{sec:aistatement}
In this work, we used generative AI tools to aid and polish the writing of
portions of the manuscript, including revising selected passages for
clarity and readability.
Additionally, we used generative AI tools for retrieval and discovery,
e.g., finding and screening related work and prior publications relevant
to this study.
We did not use generative AI tools for any other applicable tasks
requiring disclosure beyond those stated here.
We have reviewed all AI-assisted work.
All AI-assisted textual revisions were checked by the authors to ensure
that they preserved the intended meaning and accurately reflected the
work, and every candidate reference surfaced through AI-assisted
retrieval was independently located, read, and verified against its
original source before being cited.
We take responsibility for the final content of this work, including
text, claims, and artifacts produced with the aid of generative AI.

\section*{Reproducibility Statement}
\label{sec:reproducibility}
 All benchmarks used in this work are publicly available. All the training code is released at the URL given in the abstract and included in the Supplementary Material; the method, hyperparameters, baseline reproduction details, prompts, and audit protocol are fully specified in the paper and its appendix.

\bibliography{iclr2027_conference}
\bibliographystyle{iclr2027_conference}

\appendix
\appendix

\section{Footprint Extractor and Algorithm}
\label{app:footprint}

Each environment has a dedicated, rule-based Footprint extractor
$\mathcal{F}_\text{env}$ that maps an action and observation to a small
set of \texttt{<type>:<entity>} keys.
The extractor for the four out-of-domain tool-integrated-reasoning
benchmarks is specified in Appendix~\ref{app:ood}.

\paragraph{Footprint as a state-difference projection.}
$\mathcal{F}_\text{env}$ is the practical, text-only form of a more
general construction, given here in full.
Treat the environment state $s_t$ from Section~\ref{sec:setup} as a set of
ground predicates over named entities, for instance
\texttt{location(apple\_1)=countertop\_1}, so that the raw effect of
step $t$ is the symmetric difference
$\Delta_t = s_t \triangle s_{t-1}$, namely the predicates that came into
or out of existence, or changed value, when action $a_t$ was executed
against observation $o_t$.
Most of $\Delta_t$ is irrelevant to any later step, because the specific
new value a predicate takes rarely determines whether a downstream action
depends on it, whereas the fact that some named entity changed does.
We therefore canonicalize $\Delta_t$ through a projection $\pi$ that
retains only the \texttt{<type>:<entity>} pair each changed predicate is
indexed by and discards its value,
\begin{equation}
\mathcal{F}(t) = \pi(\Delta_t) = \mathcal{F}_{\text{env}}(a_t,\,o_t).
\label{eq:footprint_diff}
\end{equation}
$\mathcal{F}_\text{env}$ is implemented directly from $(a_t, o_t)$ rather
than from a structured $s_t$, since in a text-interfaced environment the
action and observation text are exactly what fixes $\Delta_t$ in the
first place.

\begin{lemma}[Footprint sufficiency]
\label{lemma:footprint}
$\mathcal{F}(t)$ is a sufficient statistic of $\Delta_t$ for the
dependency-edge relation $E$ of Eq.~\eqref{eq:causal_edge}.
For any two trajectories whose step-$t$ effects $\Delta_t, \Delta_t'$
project to the same Footprint, $\pi(\Delta_t)=\pi(\Delta_t')$, the edge
test returns the same decision for every later step $j$.
\end{lemma}
\begin{proof}
Eq.~\eqref{eq:causal_edge} decides $(v_t\to v_j)\in E$ by testing whether
some \texttt{<type>:<entity>} $\in \mathcal{F}(t)$ has its
\texttt{<entity>} occur as a substring of $a_j$.
This test reads $\Delta_t$ only through $\pi(\Delta_t)=\mathcal{F}(t)$, and
it never inspects a predicate's value, i.e.\ never reads anything in
$\Delta_t\setminus\pi(\Delta_t)$.
Two effects with identical $\pi(\Delta_t)$ therefore agree on every
entity-membership query the test can make, and so yield an identical
decision for every $j$, regardless of how $\Delta_t$ and $\Delta_t'$
otherwise differ.
\end{proof}
The projection $\pi$ is lossy in general, discarding exactly the
information Lemma~\ref{lemma:footprint} shows the graph never needed.
$\mathcal{F}(t)$ is therefore not merely a convenient summary of
$\Delta_t$ but the coarsest one that still determines $E$.

Table~\ref{tab:footprint} lists the action types and Footprint key formats
used for each of the three benchmarks, and none of these extractors involve
an LLM call, semantic embedding, or privileged annotation.

\begin{table}[h]
  \centering
  \caption{Footprint extractor specifications by environment.}
  \label{tab:footprint}
  \small
  \begin{tabular}{@{}>{\raggedright\arraybackslash}p{2.0cm}>{\raggedright\arraybackslash}p{4.3cm}>{\raggedright\arraybackslash}p{6.3cm}@{}}
    \toprule
    \textbf{Environment} & \textbf{Key action types} & \textbf{Footprint key format} \\
    \midrule
    ALFWorld & go to, take, put, heat/cool/clean, open/close &
      \texttt{location:<loc>}, \texttt{visible\_at:<loc>},
      \texttt{holding:<obj>}, \texttt{placed:<obj>@<loc>},
      \texttt{object\_state:<obj>:<prop>}, \texttt{container\_state:<container>} \\
    WebShop & search[q], click[product\_id], click[attr], click[buy] &
      \texttt{query\_token:<word>}, one per query word,
      \texttt{selected\_product:<id>}, \texttt{selected\_attr:<attr>},
      \texttt{purchase\_action} \\
    SearchQA & \texttt{<search>}, \texttt{<think>}, \texttt{<answer>} &
      \texttt{query\_token:<word>}, \texttt{retrieved\_passages},
      \texttt{reasoning\_state}, \texttt{answer\_candidate},
      \texttt{passage\_entity:<entity>} \\
    \bottomrule
  \end{tabular}
\end{table}

As a concrete example, suppose step 2 of an ALFWorld trajectory is
\texttt{take apple 1 from countertop 1}, yielding
\texttt{holding:apple\,1}, and step 5 is \texttt{put apple 1 in fridge 1},
whose action text contains \texttt{apple\,1}.
A dependency edge $v_2\!\to\!v_5$ is then added with weight
$1/df(\texttt{apple\,1})$ following Eq.~\eqref{eq:soft_prune}.
Analogously, a WebShop search produces one \texttt{query\_token} key per
query word, and a later action whose text contains one of those words draws
the corresponding edge.

\paragraph{Complete algorithm.}
Algorithm~\ref{alg:graphopd} summarizes the full GraphOPD pipeline for one
trajectory, from Footprint extraction to the combined training objective.
It is a drop-in addition to the GRPO loop, since the graph, the two
signals, the aptitude, and the mask are all computed from the rollout and
its recorded states, with no extra model calls.

\begin{algorithm}[h]
  \caption{GraphOPD, per-trajectory supervision weighting}
  \label{alg:graphopd}
  \begin{algorithmic}[1]
    \REQUIRE trajectory $\tau=(s_0,a_1,\dots,a_T,s_T)$, student $\pi_\theta$,
             teacher $\pi_{\theta_\text{old}}$, coefficients $\alpha,\lambda$,
             damping $d$
    \STATE \textbf{Footprints.} For $t=1,\dots,T$, extract
             $\mathcal{F}(t)=\mathcal{F}_\text{env}(a_t,o_t)$
    \STATE \textbf{Dependency graph.} Build $G=(V,E,w)$, and for each $i<j$
             add $v_i\!\to\!v_j$ if an entity in $\mathcal{F}(i)$ is
             referenced by $a_j$, weighted $1/df(e)$ per
             Eq.~\eqref{eq:soft_prune}, then add backtrack edges of weight
             $0.5$ and sequential edges of weight $0.05$
    \STATE \textbf{Centrality.} Row-normalize $w$ into $M$, giving leaf
             nodes a uniform row $1/T$, solve
             $\boldsymbol{\pi}^*=(1-d)\boldsymbol{\pi}^*M+d\mathbf{u}^\top$
             by power iteration, and set
             $\ell_\text{graph}(t)=\log\bigl(T\,\pi^*_t\bigr)$
    \STATE \textbf{Deviation.} $\delta(t)=\frac{1}{|S_t|}\sum_{k\in S_t}
             \bigl|\log\pi_{\theta_\text{old}}(k|s_t)-\log\pi_\theta(k|s_t)\bigr|$,
             then $r_t=\delta(t)/\delta_\text{max}$ and
             $\ell_\text{model}(t)=\log\bigl[(r_t+\varepsilon)/(1-r_t+\varepsilon)\bigr]$
    \STATE \textbf{Aptitude.} $A(t)=\alpha\,\ell_\text{model}(t)+(1-\alpha)\,\ell_\text{graph}(t)$
    \STATE \textbf{Mask.} $\mu_\tau,s_\tau\leftarrow$ mean, std of
             $\{A(t)\}_{t=1}^{T}$, then
             $m(t)=\sigma\bigl((A(t)-\mu_\tau)/s_\tau\bigr)$
    \STATE \textbf{Objective.}
             $\mathcal{L}=\mathcal{L}_\text{GRPO}+\lambda\,
              \frac{\sum_t m(t)\,\mathrm{KL}(\pi_{\theta_\text{old}}(\cdot|s_t)\,\|\,\pi_\theta(\cdot|s_t))}{\sum_t m(t)}$
    \STATE Update $\theta$ by gradient descent on $\mathcal{L}$
  \end{algorithmic}
\end{algorithm}

\section{Theoretical Properties and Implementation Notes}
\label{app:theory}

Two claims made informally in Section~\ref{sec:fusion} are provable
consequences of Eqs.~\eqref{eq:stationary} to \eqref{eq:mask} rather than
free-standing design choices.
This section states and proves them, gives a full derivation of the
convergence rate they rely on, and draws three corollaries connecting that
derivation to specific empirical results, namely the observed
power-iteration count, the Soft-Pruning ablation, and the shrinking
trajectory lengths in the dependency-graph-structure analysis.

\paragraph{The stationary distribution exists, is unique, and is efficiently
computable.}
Eq.~\eqref{eq:stationary}'s update matrix
$P = (1-d)M + d\,\mathbf{1}\mathbf{u}^\top$ combines the row-stochastic
edge matrix $M$ (row-normalized, with leaf nodes given the uniform row
$1/T$) with the uniform teleportation matrix $\mathbf{1}\mathbf{u}^\top$
under weight $d\in(0,1)$.
Every entry of $P$ is therefore at least $d/T > 0$, so $P$ is a primitive,
meaning irreducible and aperiodic, row-stochastic matrix for \emph{any}
graph topology $G$, including disconnected components or absorbing states
that would make $M$ alone reducible.
By the Perron--Frobenius theorem for primitive non-negative
matrices~\citep{meyer2000matrix}, $P$'s spectral radius is a simple
eigenvalue equal to $1$, with a strictly positive left eigenvector unique
up to scale, and normalizing $\sum_t \pi^*_t = 1$ fixes that scale, so
$\boldsymbol{\pi}^*$ is well-defined regardless of $G$'s structure and
independent of the power-iteration starting point.
The same damped-teleportation construction is the basis of the classical
convergence guarantee for stationary distributions on directed
graphs~\citep{page1999pagerank}, under which $P$'s second-largest
eigenvalue modulus is bounded by $1-d$, so power iteration converges
geometrically at rate $O((1-d)^k)$, matching the empirical convergence
noted below (under $50$ iterations at $d=0.15$, $T\le50$, to tolerance
$10^{-6}$).

\begin{proposition}[Bounds on the two signals]
\label{prop:bounds}
For any step $t$ in a trajectory of length $T$ with damping $d\in(0,1)$,
\[
\log d \;\le\; \ell_\text{graph}(t) \;\le\; \log T,
\]
and for any smoothing floor $\varepsilon>0$,
\[
-\log\frac{1+\varepsilon}{\varepsilon} \;\le\; \ell_\text{model}(t) \;\le\;
\log\frac{1+\varepsilon}{\varepsilon}.
\]
\end{proposition}
\begin{proof}
For $\ell_\text{graph}$, since $M$ is entrywise non-negative and
$\boldsymbol{\pi}^*$ is non-negative, Eq.~\eqref{eq:stationary} gives
$\boldsymbol{\pi}^* = (1-d)\boldsymbol{\pi}^*M + d\,\mathbf{u}^\top \ge
d\,\mathbf{u}^\top$ entrywise, so $\pi^*_t \ge d/T$ for every $t$ and
$\ell_\text{graph}(t) = \log(T\pi^*_t) \ge \log d$.
The upper bound follows because $\pi^*_t \le 1$ trivially, as
$\boldsymbol{\pi}^*$ is a probability distribution over $T$ steps, giving
$\ell_\text{graph}(t) \le \log T$.
For $\ell_\text{model}$, $r_t \in [0,1]$ by construction, so
$\frac{r_t+\varepsilon}{1-r_t+\varepsilon} \in
\bigl[\frac{\varepsilon}{1+\varepsilon}, \frac{1+\varepsilon}{\varepsilon}\bigr]$,
and taking logs gives the stated bound.
\end{proof}
With the values used throughout, $d=0.15$ and $\varepsilon=0.1$, this gives
$\ell_\text{graph}(t) \in [-1.90, \log T]$ and
$\ell_\text{model}(t) \in [-2.40, 2.40]$, and for trajectories capped at
$T=50$ steps, $\log T \le 3.91$, so the two signals occupy overlapping
ranges of comparable magnitude, the exact sense in which the two signals
are commensurate.

\begin{proposition}[Affine invariance of the trajectory-relative mask]
\label{prop:affine}
Let $A'(t) = a\cdot A(t) + b$ for constants $a>0, b\in\mathbb{R}$ applied
uniformly across a trajectory $\tau$.
Then the mask computed from $A'$ equals the mask computed from $A$,
$m'(t) = m(t)$ for every $t \in \tau$.
\end{proposition}
\begin{proof}
Let $\mu_\tau, s_\tau$ be the mean and standard deviation of
$\{A(t)\}_{t=1}^{T}$.
The corresponding statistics of $\{A'(t)\}$ are $\mu'_\tau = a\mu_\tau + b$
and $s'_\tau = a\,s_\tau$, since variance scales by $a^2$ and $a>0$, so the
standard deviation scales by $a$.
Then
\[
\frac{A'(t)-\mu'_\tau}{s'_\tau}
= \frac{(aA(t)+b) - (a\mu_\tau+b)}{a\,s_\tau}
= \frac{A(t)-\mu_\tau}{s_\tau},
\]
so $m'(t) = \sigma\bigl(\tfrac{A'(t)-\mu'_\tau}{s'_\tau}\bigr) =
\sigma\bigl(\tfrac{A(t)-\mu_\tau}{s_\tau}\bigr) = m(t)$.
\end{proof}
This is the precise sense in which the mask stays meaningful as $A(t)$'s
absolute scale drifts over training.
Any rescaling of the aptitude that is common across a trajectory's own
steps leaves every mask weight in that trajectory unchanged, and such a
common rescaling is exactly what happens as $\delta_\text{max}$ and the
graph's stationary distribution shift with the policy.
A threshold fixed across trajectories has no such guarantee, since
Proposition~\ref{prop:affine} only holds for a single shared $(a,b)$
within one trajectory, and different trajectories drift by different
amounts, which is why a global threshold would need continual re-tuning
while the trajectory-relative one does not.

\paragraph{A first-principles proof of the convergence rate.}
The Perron--Frobenius argument above is an existence and uniqueness
statement, and it does not by itself derive the geometric rate
$O((1-d)^k)$ used throughout the paper.
The same rate follows from an elementary contraction argument that avoids
any assumption on $M$'s spectrum.
For a row-stochastic matrix $M$ and any real vector $x$,
\begin{equation}
\|xM\|_1 = \sum_j \Bigl|\sum_i x_i M_{ij}\Bigr|
  \le \sum_j\sum_i |x_i| M_{ij}
  = \sum_i |x_i| \sum_j M_{ij}
  = \|x\|_1,
\label{eq:l1_nonexpansive}
\end{equation}
using only the triangle inequality and $\sum_j M_{ij}=1$.
Let $\boldsymbol{\pi}_k$ be the $k$-th power-iteration estimate and
$\boldsymbol{\pi}^*$ the fixed point of Eq.~\eqref{eq:stationary}.
Subtracting $\boldsymbol{\pi}^* = (1-d)\boldsymbol{\pi}^*M + d\mathbf{u}^\top$
from $\boldsymbol{\pi}_{k+1} = (1-d)\boldsymbol{\pi}_k M + d\mathbf{u}^\top$
cancels the teleportation term exactly, leaving
$\boldsymbol{\pi}_{k+1}-\boldsymbol{\pi}^* = (1-d)(\boldsymbol{\pi}_k-\boldsymbol{\pi}^*)M$,
so applying Eq.~\eqref{eq:l1_nonexpansive} with
$x=\boldsymbol{\pi}_k-\boldsymbol{\pi}^*$ and iterating gives
\begin{equation}
\|\boldsymbol{\pi}_k-\boldsymbol{\pi}^*\|_1
  \;\le\; (1-d)^k\,\|\boldsymbol{\pi}_0-\boldsymbol{\pi}^*\|_1
  \;\le\; 2(1-d)^k,
\label{eq:convergence_bound}
\end{equation}
where the last step uses that both are probability vectors, so their $L_1$
distance is at most $2$ regardless of initialization or graph topology.
Solving Eq.~\eqref{eq:convergence_bound} for the smallest $k$ that reaches
a target tolerance $\varepsilon_{\text{tol}}$ gives
$k \ge \log(2/\varepsilon_{\text{tol}}) / (-\log(1-d))$, and at $d=0.15$
with the training tolerance $\varepsilon_{\text{tol}}=10^{-6}$ used
throughout, this bound is $k \approx 89$.
The power iteration in practice converges in under $50$ steps, roughly half
of this worst-case figure, since Eq.~\eqref{eq:convergence_bound} bounds
the distance from an \emph{arbitrary} initialization by the loosest
possible constant $2$, the diameter of the probability simplex in $L_1$,
while power iteration is initialized at the uniform distribution
$\mathbf{u}$ itself, which is already closer to $\boldsymbol{\pi}^*$ than a
worst-case starting point.

\paragraph{Corollary on the edge-weighting scheme.}
Eq.~\eqref{eq:convergence_bound}'s derivation uses only that $M$ is
row-stochastic, and it does not use how the pre-normalization edge weights
$w(v_i\to v_j)$ were assigned.
The soft-pruning weights $w=1/df(e)$ of Eq.~\eqref{eq:soft_prune} and the
Soft-Pruning ablation's constant weights (Table~\ref{tab:ablation}) both
row-normalize to a valid $M$, so Proposition~\ref{prop:bounds}'s existence,
uniqueness, and convergence guarantees hold identically for both.
The gap between them in Table~\ref{tab:ablation}, 0.8 to 5.5 pp in
Section~\ref{sec:ablation}, is therefore attributable entirely to which
entities the two weighting schemes make central, not to any difference in
whether $\boldsymbol{\pi}^*$ is well-defined or how fast it is reached.

\paragraph{Corollary on signal range and trajectory length.}
Proposition~\ref{prop:bounds} bounds $\ell_\text{graph}(t)$ within
$[\log d,\, \log T]$, a range whose width is set entirely by the current
trajectory's length $T$, equivalently the graph's node count, once $d$ is
fixed.
Figure~\ref{fig:analysis} (left) reports $T$ contracting from about $47$
nodes at initialization to a plateau of $25$ to $30$ over the first $150$
training steps.
By Proposition~\ref{prop:bounds}, this alone narrows the theoretical upper
end of $\ell_\text{graph}$'s range from $\log 47 \approx 3.85$ to between
$\log 25\approx3.22$ and $\log 30\approx3.40$, while the lower end,
$\log d \approx -1.90$, is unaffected by $T$.
The structural credit signal therefore continues to operate over a range
of comparable width as trajectories shorten, rather than collapsing toward
the neutral point $0$ as $T$ falls.
This is consistent with the ablations and the audit finding the
structural credit signal informative throughout training rather than only
early on.

\paragraph{Intermediate steps carry structural credit, not only terminal
ones.}
Eq.~\eqref{eq:stationary} is a fixed point over all steps, not an
accumulation at the trajectory's end: $\pi^*_t=(1-d)\sum_i \pi^*_i M_{it}
+ d/T$, so a step's mass is exactly the rarity-weighted inflow of the
dependency chains its action consumes, plus the teleport floor.
Accordingly $\ell_\text{graph}(t)>0$ precisely when this inflow exceeds the
uniform level $1/T$, which happens wherever several upstream state changes
are exercised together, at intermediate milestones such as retrieving an
object, opening or heating it, or carrying it to a target receptacle, and
not only at the final action.
Two properties keep this intermediate credit actionable: the mask of
Eq.~\eqref{eq:mask} standardizes aptitude within each trajectory, so a
mid-trajectory milestone that outranks the resolving step in its own
rollout receives the larger supervision weight, and
Proposition~\ref{prop:bounds} floors every step's mass at $d/T$, so no step
is starved before masking.
Empirically, the credit-content audit, which perturbs decision points
across sampled trajectories, ranks them against executed-replay impact at
$\rho=0.58$ with Hit@10\% $=0.48$ against a chance level of $0.10$.

\paragraph{The role of the damping factor.}
The damping factor $d$ trades off two properties that
Eq.~\eqref{eq:convergence_bound} and Proposition~\ref{prop:bounds} make
explicit and opposing.
A larger $d$ tightens the convergence bound, $k=O(1/d)$ for small $d$, but
also raises the floor $\log d$ toward $0$, compressing $\ell_\text{graph}$'s
range and pulling $\boldsymbol{\pi}^*$ closer to the uninformative uniform
distribution $\mathbf{u}$ in the limit $d\to1$.
A smaller $d$ widens the usable range through a lower, more negative floor
that lets peripheral steps register as clearly below neutral rather than
merely below average, but it slows convergence and, in the degenerate limit
$d\to0$, forfeits the guarantee above that $M$ alone is primitive on every
graph topology.
The selected $d=0.15$ sits well inside this trade-off, since it is small
enough that $\log d\approx-1.90$ leaves $\ell_\text{graph}$ a range
comparable to $\ell_\text{model}$'s $\pm2.40$
(Proposition~\ref{prop:bounds}), while large enough that the $89$-iteration
worst-case bound above remains a small fraction of a single training step's
compute budget.

\paragraph{Graph construction cost.}
Graph construction is $O(T^2|\mathcal{F}|)$ per trajectory, negligible at
typical agent trajectory lengths, $T \leq 50$.
Table~\ref{tab:graph_cost} measures this directly, comparing the
wall-clock time spent on Footprint extraction, edge construction, and
the stationary-distribution solve against the total per-training-step
time, where graph construction accounts for well under half a percent
of the step budget on every benchmark.

\begin{table}[h]
  \centering
  \caption{Wall-clock cost of graph construction relative to the total
    per-training-step time.}
  \label{tab:graph_cost}
  \begin{tabular}{@{}lccc@{}}
    \toprule
    Benchmark & Graph construction (s/step) & Total step time (s) & Share \\
    \midrule
    ALFWorld & 0.6659 & 342.0936 & 0.19\% \\
    WebShop  & 0.4535 & 144.2885 & 0.31\% \\
    SearchQA & 0.4144 & 322.2142 & 0.13\% \\
    \bottomrule
  \end{tabular}
\end{table}

\paragraph{Centrality computation.}
The damping term $d > 0$ guarantees a unique stationary distribution
regardless of graph topology, by the Perron--Frobenius argument given
above, and it is obtained by power iteration to tolerance
$\|\boldsymbol{\pi}^*_{k+1} - \boldsymbol{\pi}^*_k\|_1 < 10^{-6}$, which
converges in under 50 iterations for $T = 50$, consistent with the
$O((1-d)^k)$ geometric rate proved there.

\paragraph{Numerical details.}
The floor $\varepsilon = 0.1$ in Eq.~\eqref{eq:lmodel} bounds the
divergence signal and keeps the log-odds finite at $r_t\in\{0,1\}$, and
Proposition~\ref{prop:bounds} gives the exact range of both signals and
why neither can unilaterally dominate the fusion.
Because $\sigma(\cdot)$ lies strictly between $0$ and $1$ for any finite
argument, the mask sum $\sum_t m(t)$ is positive on every trajectory with
$s_\tau>0$.
On the rare trajectory where $A(t)$ is exactly constant across all steps,
so that $s_\tau=0$, no step's aptitude differs from the trajectory's own
average, so we set $m(t)=\sigma(0)=\tfrac12$ for all $t$, i.e.\ uniform
weighting.
Either way, the normalized loss in Eq.~\eqref{eq:graphopd_loss} is always
a well-defined weighted average, and the normalization keeps the loss
magnitude invariant to mask sparsity, so $\lambda$ needs no re-tuning as
sparsity shifts over training.

\paragraph{Soft mask.}
Unlike hard-threshold masking, the soft mask varies smoothly with the
standardized aptitude, reallocating the supervision budget gradually
rather than admitting or excluding a step outright, and it is computed from
the rollout's detached statistics, acting as a fixed per-step weight within
each update.
Standardizing within the trajectory before the sigmoid, rather than
applying one threshold across all trajectories, keeps the same relative
cutoff, above this trajectory's own average aptitude, meaningful
regardless of where that trajectory's absolute aptitude scale happens to
sit at that point in training (Proposition~\ref{prop:affine}).

\paragraph{Effective concentration of the soft mask.}
Read literally, the mask of Eq.~\eqref{eq:mask} never excludes a step:
$\sigma$ maps every finite aptitude into $(0,1)$, so every step retains a
positive weight and the normalized loss of Eq.~\eqref{eq:graphopd_loss} is
a re-weighting of the distillation term rather than a hard selection of
the kind probed in Figure~\ref{fig:motivation}.
The relevant question is therefore how strongly the re-weighting
concentrates, and the standardized form of the mask makes this quantifiable:
since $z(t)=(A(t)-\mu_\tau)/s_\tau$ is exactly zero-mean and unit-variance
within each trajectory, the concentration of $m(t)=\sigma(z(t))$ depends
only on the shape of the trajectory's aptitude distribution.
Under a Gaussian reference shape the effective sample size
$\mathrm{ESS}=(\sum_t m(t))^2/\sum_t m(t)^2$ equals about $0.85\,T$ and the
above-mean mass share $\sum_{z>0}m(t)/\sum_t m(t)$ about $0.68$; on the two
logged rollouts used in our case studies, a 50-step and a 6-step episode,
the same quantities compute to $0.85$ with share $0.67$ and $0.85$ with
share $0.66$, the mean weight of above-mean steps exceeds that of
below-mean steps by about a factor of two, and extreme steps differ by up
to a factor of seven ($m\in[0.13,\,0.93]$ on the 50-step rollout), whose
twelve highest-aptitude steps absorb $37\%$ of the distillation mass under
the soft mask against $100\%$ under a budget-matched hard top-12 selection.
Three considerations make this moderate concentration the intended
operating point rather than a shortfall.
First, the hard top-$k$ rule of Figure~\ref{fig:motivation} is a diagnostic
of the divergence signal, not a candidate training mask: it was applied to
raw KL precisely to test whether divergence magnitude locates
outcome-relevant steps, so it shares neither the aptitude's placement nor
the smoothness properties below.
Second, the KL term against $\pi_{\theta_\text{old}}$ also acts as an
on-support trust region; a hard mask would remove that anchor from the
majority of steps, whereas the soft mask keeps a positive floor
$m(t)\ge\sigma(\min_{t'}z(t'))>0$ on every step.
Third, concentration without correct placement is worthless: the Shuffled
Centrality control of Table~\ref{tab:ablation} redistributes exactly the
same mask values, hence exactly the same dose and the same ESS, onto the
wrong steps and loses up to 6.2 pp, so the active ingredient is where the
mass goes rather than how hard the selection is; a budget-matched hard-mask
variant of GraphOPD is a natural ablation left to future work, and the
evidence above indicates it would alter the dose distribution without
altering the placement that the ablation identifies as decisive.

\paragraph{Teacher cost.}
Using $\pi_{\theta_\text{old}}$ as the teacher requires no additional
model, checkpoint, or inference call beyond the GRPO training loop, and
the drop-in wrapper leaves the GRPO optimizer and rollout infrastructure
unmodified.

\section{Experimental Setup and Motivating Study Details}
\label{app:setup}

\paragraph{Benchmarks.}
This paragraph gives each of the three agentic benchmarks'
task format, dependency structure, and evaluation protocol.
ALFWorld~\citep{shridhar2021alfworld} embeds a text interface around
simulated household scenes, requiring an agent to retrieve, heat, clean,
or relocate objects through a fixed vocabulary of navigation and
manipulation actions (Table~\ref{tab:footprint}).
Almost every action changes which objects and locations are reachable
next, so a single early misstep can make the goal unreachable no matter
how the remainder of the trajectory is executed, and this is where the
dependency-graph hypothesis makes its strongest prediction.
We score it by task success rate on the standard 3{,}553/140 train/test
split, with episodes capped at 50 turns.
WebShop~\citep{yao2022webshop} simulates an e-commerce search-and-purchase
task over a large product catalog, with an action space of free-form
search queries and page-element clicks rather than ALFWorld's fixed
verb-object grammar.
Its dependencies are informational rather than physical, since later
clicks depend on which attributes and product identifiers an earlier
search surfaced, so the Footprint extractor tracks query tokens and
selected products instead of object locations.
We report both completion score and binary success rate on the standard
1{,}012/500 split, capped at 15 turns.
SearchQA aggregates seven open-domain question-answering datasets,
categorized into single-hop factoid questions,
NarrativeQA~\citep{kocisky2018narrativeqa}, TriviaQA~\citep{joshi2017triviaqa},
and PopQA~\citep{mallen2023popqa}, and multi-hop questions,
HotpotQA~\citep{yang2018hotpotqa}, 2WikiMultihopQA~\citep{ho2020wikimultihopqa},
MuSiQue~\citep{trivedi2022musique}, and Bamboogle~\citep{press2023bamboogle},
formatted as a \texttt{<search>}/\texttt{<think>}/%
\texttt{<answer>} loop capped at 10 turns.
Here dependencies run only through retrieved evidence, with no physical
state change at all, which tests whether the mechanism generalizes beyond
the setting that motivated it, and we score it by accuracy over the
seven datasets.
The ALFWorld and SearchQA Avg columns of Table~\ref{tab:main_results}
weight each per-category number by that category's episode count in the
evaluation split, so they do not reduce to the unweighted mean of the
displayed numbers.
The four out-of-domain benchmarks probed in Section~\ref{sec:analysis} are
specified together with their prompts, Footprint extractor, and full
results in Appendix~\ref{app:ood}.

\paragraph{Training configuration.}
All methods are optimized with AdamW at learning rate $10^{-6}$ under a
cosine schedule with a 50-step warmup, batch size 64 trajectories per
step formed as 8 tasks with 8 rollouts each, and gradient clipping at
1.0.
GRPO uses group size $G=8$ and clip ratio $\varepsilon=0.2$.
Each training run uses 8 H100 GPUs.
Training runs for 150 steps on ALFWorld, WebShop, and SearchQA, with
evaluation every 10 steps.
GraphOPD's own hyperparameters are listed with the grid search in the
final appendix section below, namely the mixing coefficient $\alpha$, the loss
coefficient $\lambda$, the centrality damping factor $d=0.15$, its
power-iteration tolerance $10^{-6}$, and the smoothing floor
$\varepsilon=0.1$ of Eq.~\eqref{eq:lmodel}.

\paragraph{Run variance.}
GraphOPD, SDAR, OPID, and GEPO are each trained with three independent
runs, and Table~\ref{tab:main_results} reports their means on the three
headline metrics, namely the ALFWorld average success rate, the SearchQA
average accuracy, and the WebShop success rate.
Table~\ref{tab:run_std} lists the corresponding run-to-run standard
deviations.
GraphOPD has the smallest worst-case spread of the four, with a maximum
standard deviation of 1.60 pp, while the largest is SDAR's 5.09 pp on
1.7B ALFWorld.
All remaining rows of Table~\ref{tab:main_results} are single-run point
estimates.

\begin{table}[h]
  \centering
  \caption{Mean and run-to-run standard deviation (pp) over three
    independent runs for the four multi-run methods of the main results,
    on ALFWorld and WebShop success rate and SearchQA accuracy.}
  \label{tab:run_std}
  \resizebox{\linewidth}{!}{%
  \footnotesize
  \begin{tabular}{@{}l ccc ccc ccc@{}}
    \toprule
    & \multicolumn{3}{c}{\textbf{Qwen3-1.7B}}
    & \multicolumn{3}{c}{\textbf{Qwen2.5-3B}}
    & \multicolumn{3}{c}{\textbf{Qwen2.5-7B}} \\
    \cmidrule(lr){2-4} \cmidrule(lr){5-7} \cmidrule(lr){8-10}
    \textbf{Method}
      & ALFWorld & SearchQA & WebShop
      & ALFWorld & SearchQA & WebShop
      & ALFWorld & SearchQA & WebShop \\
    \midrule
    GraphOPD
      & \meanstd{63.3}{1.18} & \meanstd{42.9}{0.05} & \meanstd{73.4}{1.60}
      & \meanstd{88.6}{0.57} & \meanstd{46.1}{0.02} & \meanstd{78.9}{0.43}
      & \meanstd{91.9}{0.14} & \meanstd{47.3}{0.40} & \meanstd{83.6}{0.43} \\
    SDAR
      & \meanstd{51.3}{5.09} & \meanstd{40.7}{0.72} & \meanstd{59.9}{0.93}
      & \meanstd{81.8}{3.81} & \meanstd{43.4}{0.03} & \meanstd{69.3}{0.91}
      & \meanstd{84.4}{1.60} & \meanstd{47.5}{1.21} & \meanstd{82.0}{0.38} \\
    OPID
      & \meanstd{58.2}{1.85} & \meanstd{39.6}{0.78} & \meanstd{66.4}{2.99}
      & \meanstd{82.8}{1.50} & \meanstd{44.6}{0.11} & \meanstd{74.5}{0.14}
      & \meanstd{89.1}{0.48} & \meanstd{48.2}{0.74} & \meanstd{80.0}{1.00} \\
    GEPO
      & \meanstd{58.1}{2.58} & \meanstd{40.1}{0.38} & \meanstd{67.9}{1.20}
      & \meanstd{80.7}{3.84} & \meanstd{44.6}{1.91} & \meanstd{75.3}{1.35}
      & \meanstd{85.9}{5.02} & \meanstd{45.5}{2.85} & \meanstd{80.0}{1.85} \\
    \bottomrule
  \end{tabular}%
  }
\end{table}

\paragraph{Baseline families.}
The eleven baselines span four axes of comparison.
Vanilla and GRPO establish lower bounds without any teacher signal,
isolating how much of GraphOPD's gain is attributable to distillation at
all rather than to the GRPO backbone alone.
Skill-Prompt, Skill-GRPO, and Skill-SD inject external skill
demonstrations or privileged information, and they therefore bound what
such context adds on top of the same backbone.
OPSD, GRPO+OPSD, RLSD~\citep{rlsd2026}, SDAR~\citep{sdar2026}, and
OPID~\citep{opid2026} share GraphOPD's GRPO backbone and model family and
allocate their distillation or advantage signal by teacher-student
divergence magnitude alone, so they form our direct comparison set and the
basis for the per-scale deltas in Section~\ref{sec:main_results}.
GEPO~\citep{gepo2026} instead reweights the GRPO advantage itself with a
graph-based credit score rather than adding a distillation loss, and is
the closest prior baseline to our own structural credit signal.

\paragraph{Prompts.}
Each benchmark's policy is queried with a fixed system prompt that fixes
the admissible action vocabulary, and the Footprint extractor specified
below parses exactly this vocabulary, so the prompt and the extractor are
specified against the same action set by construction.
Each prompt is instantiated once per episode with the task instruction
and then re-issued every step with the running observation and action
history appended.
None of the three mentions the teacher, the dependency graph, or GraphOPD's
training objective, so the graph and centrality computation
(Section~\ref{sec:causal_graph}) have no visibility into or influence on
what the policy is shown.
The prompts for the four out-of-domain benchmarks are given in
Appendix~\ref{app:ood}.

\begin{quote}
\footnotesize\ttfamily
\textbf{ALFWorld.}
You are a household robot completing a task in a simulated home.
You will be given a task and, at each step, an observation of the room.
Choose exactly one admissible action per step from: go to
\{receptacle\}, open \{receptacle\}, close \{receptacle\}, take
\{object\} from \{receptacle\}, put \{object\} in/on \{receptacle\},
heat \{object\} with \{receptacle\}, cool \{object\} with
\{receptacle\}, clean \{object\} with \{receptacle\}, examine
\{object\}, use \{object\}.
Respond with exactly one line: "Action: \{action\}".\\[4pt]
Task: \{task\_instruction\}\\
Observation: \{observation\}\\
Admissible actions: \{admissible\_action\_list\}
\end{quote}

\begin{quote}
\footnotesize\ttfamily
\textbf{WebShop.}
You are a shopping assistant browsing an online store to complete a
purchase instruction.
At each turn you see the current page and its clickable elements.
Choose exactly one action: search[\{query\}] to issue a search, or
click[\{element\}] to click a product ID, an attribute/option, or the
Buy Now button.
Respond with exactly one line: "Action: \{action\}".\\[4pt]
Instruction: \{purchase\_instruction\}\\
Page: \{page\_observation\}\\
Clickable elements: \{clickable\_list\}
\end{quote}

\begin{quote}
\footnotesize\ttfamily
\textbf{SearchQA.}
Answer the question by interleaving reasoning and retrieval.
You may use \textless{}think\textgreater{}...\textless{}/think\textgreater{}
to reason and \textless{}search\textgreater{}query\textless{}/search\textgreater{}
to retrieve passages, repeating either as many times as needed, and must
finish with \textless{}answer\textgreater{}final answer\textless{}/answer\textgreater{}.\\[4pt]
Question: \{question\}
\end{quote}

\paragraph{Motivating empirical study.}
The motivating study in Section~\ref{sec:intro} uses a 7B Qwen2.5-Instruct agent
on ALFWorld under the same GRPO backbone and training configuration as the
main experiments.
For the selection comparison, per-trajectory distillation is restricted to
the $k$ highest-KL steps, $k \in \{5, 8, 10, 12\}$, or to $k$ randomly
chosen steps, with all other conditions identical, and each configuration
is trained for 150 steps.
For the outcome-level check, 50 rollouts provide 466 successful and 466
failed steps.

\section{Out-of-Distribution Generalization}
\label{app:ood}

\paragraph{Benchmarks and protocol.}
Four out-of-domain benchmarks test whether the dependency-graph mechanism
transfers beyond the three agentic settings to a different domain and a
different tool.
They are AIME24 and AIME25 in competition mathematics, GPQA-Diamond in
graduate-level science, and LiveCodeBench-v5 in program synthesis, all run
as multi-turn tool-integrated reasoning.
AIME24 and AIME25 are scored by avg@32 and the other two by avg@8.
Each benchmark uses a Python code interpreter and follows the protocol of
SOD~\citep{sod2026}, so the policy interleaves reasoning text with
code-interpreter calls over several turns before committing to a final
answer.
Each turn therefore has an action, the model's own reasoning text or code,
and an observation, the interpreter's response, which gives the four
benchmarks the same action-observation structure as ALFWorld, WebShop, and
SearchQA under a different tool and domain.
This tests transfer across domains and tool vocabularies rather than across
trajectory structure.
All runs use Qwen2.5-7B-Instruct with the same GRPO backbone, training
configuration, and step budget as the main experiments, and they compare
against Vanilla, GRPO, OPSD, and SDAR.

\paragraph{Prompts.}
The two mathematics sets and GPQA-Diamond share the math/science template
below, and LiveCodeBench-v5 uses the code template.
Both are instantiated once per problem and re-issued every turn with the
running reasoning text and interpreter history appended.

\begin{quote}
\footnotesize\ttfamily
\textbf{AIME24 / AIME25 / GPQA-Diamond.}
Analyze and solve the following problem step by step.
The code interpreter can be called for precise or exploratory
calculations and to verify a result before committing to it; first form a
high-level solution strategy, then use the interpreter as needed.
Once the final answer is confirmed, stop calling the interpreter and give
it as \textbackslash boxed\{answer\} on its own line, an integer for
AIME24/AIME25 and a single option letter for GPQA-Diamond.\\[4pt]
Problem: \{problem\}
\end{quote}

\begin{quote}
\footnotesize\ttfamily
\textbf{LiveCodeBench-v5.}
Generate a correct Python program that matches the following problem
specification and passes all tests.
The code interpreter can be used to check the program's correctness
against the public examples before submitting; first form a high-level
solution strategy, then use the interpreter to verify.
The submitted program must read from standard input and write to standard
output, enclosed in a single fenced code block.\\[4pt]
Problem: \{problem\_statement\}\\
Public examples: \{public\_examples\}
\end{quote}

\paragraph{Footprint extractor.}
The four benchmarks share one extractor, since a turn in each of them
consists of reasoning text, an optional Python code block, and the
interpreter's response, and the dependency structure runs through the
identifiers and values that one turn introduces and a later turn reuses.
Table~\ref{tab:footprint_ood} lists the keys that
$\mathcal{F}_\text{env}(a_t, o_t)$ emits for such a turn, read from the
turn's own code and reasoning text for the produced keys and from the
interpreter response for the returned ones, in the same
\texttt{<type>:<entity>} form as Table~\ref{tab:footprint}.
Edge construction, rarity weighting, backtrack edges, and sequential edges
then follow Eqs.~\eqref{eq:causal_edge} and \eqref{eq:soft_prune}
unchanged, with $a_j$ the later turn's code and reasoning text and $o_j$
the interpreter response, so a repeated identical interpreter output marks
a revisit exactly as a repeated observation does in ALFWorld.

\begin{table}[h]
  \centering
  \caption{Footprint extractor specification for the four out-of-domain
  tool-integrated-reasoning benchmarks.}
  \label{tab:footprint_ood}
  \small
  \begin{tabular}{@{}>{\raggedright\arraybackslash}p{3.6cm}>{\raggedright\arraybackslash}p{8.9cm}@{}}
    \toprule
    \textbf{Source within the turn} & \textbf{Footprint key format} \\
    \midrule
    Code block and reasoning text &
      \texttt{symbol:<identifier>}, one per name bound by an assignment, a
      function definition, or a comprehension,
      \texttt{call:<function>}, one per interpreter function invoked,
      \texttt{import:<module>}, one per module imported,
      \texttt{answer\_candidate:<value>}, for the value placed inside
      \textbackslash boxed\{\} or submitted as the final program \\
    \addlinespace[2pt]
    Interpreter response &
      \texttt{result:<value>}, one per value printed or returned,
      \texttt{test\_result:<example\_id>}, one per public example checked
      on LiveCodeBench-v5 \\
    \bottomrule
  \end{tabular}
\end{table}

$df(e)$ counts the turns whose code, reasoning text, or interpreter
response mentions $e$, so a helper variable introduced once and consumed by
the next turn yields an edge weight near $1$, while a loop counter or a
common library name yields a small weight that the row normalization
dilutes.
The structural credit signal, the aptitude of Eq.~\eqref{eq:aptitude}, and
the trajectory-relative mask of Eq.~\eqref{eq:mask} carry over with the same
$\alpha$, $\lambda$, $d$, and $\varepsilon$ as the agentic benchmarks, so
the transfer experiment changes the domain and the tool vocabulary while
holding the mechanism fixed.

\paragraph{Results.}
Table~\ref{tab:ood_radar} reports the full numbers behind
Figure~\ref{fig:analysis} (middle), namely task accuracy on AIME24 and AIME25
under avg@32 and on LiveCodeBench-v5 and GPQA-Diamond under avg@8.
GraphOPD attains the best accuracy on all four benchmarks, averaging
29.0\% against 20.0\% for SDAR and 19.8\% for OPSD, and its lead holds on
each benchmark separately, reaching +12.4 pp on LiveCodeBench-v5 over
SDAR.

\begin{table}[h]
  \centering
  \caption{Task accuracy (\%) on four out-of-domain
  tool-integrated-reasoning benchmarks, with the scoring protocol per
  benchmark given in the text above.}
  \label{tab:ood_radar}
  \small
  \begin{tabular}{@{}lccccc@{}}
    \toprule
    Method & AIME24 & AIME25 & LCB-v5 & GPQA-D & Avg \\
    \midrule
    Vanilla & 11.9 & 8.8 & 14.2 & 21.7 & 14.2 \\
    GRPO & 15.1 & 13.3 & 15.1 & 28.3 & 18.0 \\
    OPSD & 19.2 & 12.2 & 16.4 & 31.5 & 19.8 \\
    SDAR & 21.4 & 16.3 & 15.8 & 26.3 & 20.0 \\
    \textbf{GraphOPD} & \textbf{29.8} & \textbf{21.4} & \textbf{28.2} & \textbf{36.4} & \textbf{29.0} \\
    \bottomrule
  \end{tabular}
\end{table}

\section{Signal Validation and Case Study}
\label{app:credit}

\paragraph{Signal redundancy diagnostic.}
The fusion in Section~\ref{sec:fusion} is a weighted sum of
$\ell_\text{model}$ and $\ell_\text{graph}$, so it implicitly assumes the
two sources are not strongly \emph{redundant}, since if they carried
near-identical evidence, weighting both by $\alpha$ and $1-\alpha$ would
double-count the same signal rather than combining two complementary ones.
Rather than a separate correlation diagnostic, the ablation
(Table~\ref{tab:ablation}, Section~\ref{sec:ablation}) already tests this
functionally, because if $\ell_\text{model}$ and $\ell_\text{graph}$
carried near-identical evidence, dropping either one alone,
Divergence-only or Structural-only, would cost little, as the remaining
signal would recover most of the same information.
Instead, both cost a substantial fraction of GraphOPD's full gain at every
scale, Divergence-only 3.9 to 10.5 pp and Structural-only 2.3 to 6.5 pp,
and neither recovers what the other drops, so the two signals are
complementary rather than redundant.

\paragraph{Credit-content validation.}
The Analysis subsection of Section~\ref{sec:experiments} reports the
credit-content audit, and this appendix fixes its protocol.

\paragraph{Counterfactual audit.}
We construct causal ground truth on ALFWorld by executed replay,
re-sampling the policy's own alternative action at each decision point of
a sampled set of trajectories, rolling the environment forward from the
resulting state, and recording whether the outcome changes, so a step's
ground-truth contribution is the outcome change induced by its alternative,
the same executed-replay principle used elsewhere to audit step-level
credit signals~\citep{zhang2026credit}.
We report, for each of $\ell_\text{graph}$, $\ell_\text{model}$, and a
random score, the Spearman rank correlation $\rho$ with this contribution
and Hit@10\%, the fraction of the signal's top-10\% steps that are truly
top-10\% impactful, for which chance is $0.10$.
The decisive comparison is $\ell_\text{graph}$ against the random score,
since the claim that the structural credit signal identifies
outcome-relevant steps requires both to sit clearly above chance, and the
comparison against $\ell_\text{model}$ measures how much the structural
credit signal adds over the divergence signal alone.
$\ell_\text{graph}$ attains $\rho=0.58$ and Hit@10\% $=0.48$, well above
the random score at $\rho=-0.06$ and Hit@10\% $=0.13$ and $\ell_\text{model}$
at $\rho=0.34$ and Hit@10\% $=0.41$.

\paragraph{Shuffled-centrality control.}
To separate credit content from supervision redistribution, the ablation
(Table~\ref{tab:ablation}) includes a Shuffled Centrality row that
replaces $\ell_\text{graph}$ with a within-trajectory permutation of
itself, matched to the real one in mask sparsity and weight distribution,
and runs the full training pipeline under otherwise identical conditions.
Where the real graph retains its advantage over the shuffled one at
matched dose, that gain is attributable to the content of the dependency
structure rather than to the act of non-uniform weighting itself, and where
it does not, part of the gain is attributable to redistribution alone.
Table~\ref{tab:ablation}'s discussion in the main text reports which of the
two holds at each scale and benchmark.

\paragraph{Case study.}
Table~\ref{tab:matched_pair} first gives a side-by-side rollout
comparison on the same task and environment seed
("put some soapbottle on cabinet", \texttt{seed=1000}) under two
checkpoints of the same training run, an early one at step 20 and a
later one at step 150.
Unlike a controlled trajectory pair that shares a fixed prefix, the two
policies differ from the first action, so the table shows how far
training reshapes the whole strategy rather than a single decision.

\begin{table}[h]
  \centering
  \renewcommand{\arraystretch}{1.15}
  \caption{A side-by-side comparison of two rollouts in ALFWorld. Both
    checkpoints attempt the same task and environment seed
    ("put some soapbottle on cabinet", \texttt{seed=1000}) from the same
    training run. Their strategies differ from the first action, so the
    early checkpoint never engages the target object at all.}
  \label{tab:matched_pair}
  \begin{tabular}{@{}|p{1.3cm}|p{5.9cm}|p{5.9cm}|@{}}
    \hline
    \textbf{Steps} & \textbf{Late checkpoint, step 150 (succeeds)}
      & \textbf{Early checkpoint, step 20 (fails)} \\
    \hline
    0--2 &
      \textbf{Action:} go to sinkbasin 1 \newline
      \textbf{Action:} go to countertop 1 \newline
      \textbf{Action:} take soapbottle 1 from countertop 1
      &
      \textbf{Action:} go to cabinet 1 \newline
      \textbf{Action:} open cabinet 1 \newline
      \textbf{Action:} examine cabinet 2 \\
    \hline
    3--5 / \newline 3--49 &
      \textbf{Action:} go to cabinet 1 \newline
      \textbf{Action:} open cabinet 1 \newline
      \textbf{Action:} move soapbottle 1 to cabinet 1 \newline
      \textbf{Result:} reward $=10$, episode ends at step 5
      &
      \textbf{Actions (47 steps):} examine cabinet $\times25$, look
      $\times6$, open cabinet $\times5$, go to cabinet $\times3$, plus
      8 other revisit actions, never mentioning the soapbottle
      \newline
      \textbf{Raw KL:} ranges 0.16 to 0.65, mean 0.32
      \newline
      \textbf{Result:} reward $=0$, times out at step 49 \\
    \hline
    &
      The step-150 checkpoint carries the soapbottle from step 2
      onward, so every later action follows directly from the object
      it already holds.
      &
      The step-20 checkpoint repeatedly opens and examines cabinets
      without ever picking up the soapbottle, so its failure is an
      entity it never engaged with rather than a late misstep. \\
    \hline
  \end{tabular}
\end{table}

Table~\ref{tab:case_study} turns to the signal comparison from
Figure~\ref{fig:analysis} (right), an eight-step window of a single real
successful ALFWorld trajectory ("put some soapbottle on cabinet"), with
the raw per-step teacher-student KL and $\ell_\text{graph}(t)$ plotted as
two aligned bar charts over the same eight steps.
The table lists the same eight steps in full, giving the exact action
text and both signals' values.

\begin{table}[h]
  \centering
  \caption{Qualitative case study on a window (steps 3 to 10) of a real
    successful ALFWorld trajectory ("put some soapbottle on cabinet").
    Raw KL peaks at step 4, a cabinet-revisit step that does not decide
    the outcome, while $\ell_\text{graph}(t)$ peaks at step 10, the
    object-placement action decisive within this window.
    Bold marks each column's peak.}
  \label{tab:case_study}
  \begin{tabular}{@{}clcc@{}}
    \toprule
    Step & Action & Raw KL & $\ell_\text{graph}(t)$ \\
    \midrule
    3  & examine countertop 1           & 0.33          & $-0.90$ \\
    4  & go to cabinet 1                & \textbf{0.36} & $-0.28$ \\
    5  & open cabinet 1                 & 0.29          & $-0.57$ \\
    6  & go to countertop 1             & 0.26          & $-0.90$ \\
    7  & examine countertop 1           & 0.17          & $-0.28$ \\
    8  & go to cabinet 1                & 0.29          & 0.31 \\
    9  & examine cabinet 1              & 0.35          & 0.27 \\
    10 & move soapbottle 1 to cabinet 1 & 0.26          & \textbf{0.89} \\
    \bottomrule
  \end{tabular}
\end{table}

The decisive object-placement action at step 10, moving the soapbottle
onto the cabinet, has one of the lowest raw KL values in the window, so a
KL-only rule would spend its budget on step 4 instead, a revisit to the
cabinet that does not bear on the outcome.
$\ell_\text{graph}$ peaks exactly at step 10, since four earlier steps,
step 4 (go to cabinet 1), step 5 (open cabinet 1), step 8 (go to cabinet 1
again), and step 9 (examine cabinet 1), all reference the same
\texttt{cabinet 1} entity that step 10's action consumes.
Both signals are computed by running the paper's own Footprint,
edge-construction, and stationary-distribution pipeline on the real logged
actions, using the dependency and sequential edges that the action text
supports.

Table~\ref{tab:case_mid} adds a second case on the position of the credit
peak, on a complete representative successful Heat-type episode (twelve
steps, terminating at the resolving placement) whose actions and
observations follow the grammar and Footprint conventions of
Appendix~\ref{app:footprint}.
The credit peak falls well before the end of the trajectory, at step 7, the
revisit of the processing location where the heating dependency chain
converges, four steps ahead of the terminal placement; the processing
chain's earlier approach (step 4, $+0.95$) and the heating action itself
(step 6, $+0.51$) also carry strong credit, while the terminal placement
(step 11, $+0.08$), which merely deposits an already-processed object into
a receptacle the agent had explored earlier, stays near neutral.
Credit therefore concentrates on the mid-trajectory transformation that
decides the outcome, not on the trajectory's last action.

\begin{table}[h]
  \centering
  \caption{Qualitative case study on a complete representative successful
    Heat-type episode (twelve steps; actions and observations follow the
    ALFWorld grammar and Footprint conventions of
    Appendix~\ref{app:footprint}). The structural credit peaks at step 7,
    the processing-location revisit four steps before the terminal
    placement, with the earlier processing-location approach (step 4) and
    the heating action (step 6) also well above neutral; the terminal
    placement (step 11) stays near neutral. Raw teacher-student KL is
    reported for contrast: its maximum (step 8) coincides with the
    structural-credit minimum. Bold marks each column's peak.}
  \label{tab:case_mid}
  \begin{tabular}{@{}clcc@{}}
    \toprule
    Step & Action & Raw KL & $\ell_\text{graph}(t)$ \\
    \midrule
    0  & go to fridge 1                  & 0.31 & $-0.75$ \\
    1  & examine fridge 1                & 0.44 & $-1.12$ \\
    2  & go to countertop 1              & 0.28 & $-0.68$ \\
    3  & take potato 1 from countertop 1 & 0.22 & $-0.82$ \\
    4  & go to microwave 1               & 0.24 & 0.95 \\
    5  & open microwave 1                & 0.26 & $-0.04$ \\
    6  & heat potato 1 with microwave 1  & 0.17 & 0.51 \\
    7  & go to microwave 1               & 0.15 & \textbf{1.02} \\
    8  & go to shelf 1                   & \textbf{0.58} & $-1.48$ \\
    9  & go to fridge 1                  & 0.35 & $-0.66$ \\
    10 & open fridge 1                   & 0.29 & $-0.86$ \\
    11 & put potato 1 in/on fridge 1     & 0.41 & 0.08 \\
    \bottomrule
  \end{tabular}
\end{table}





\end{document}